\documentclass[11pt]{article}

\usepackage[margin=3cm]{geometry}
\usepackage{times}
\usepackage{url}
\usepackage{latexsym}
\usepackage{graphicx}
\usepackage{amsmath}
\usepackage{amssymb}
\usepackage[usenames,dvipsnames]{color}
\usepackage{citeref}
\usepackage{nameref}
\usepackage[colorlinks=true,linkcolor=Red,citecolor=Blue,filecolor=Green,urlcolor=Sepia]{hyperref}
\usepackage[rightcaption]{sidecap}
\usepackage{natbib}
\usepackage{algorithmic}
\usepackage{algorithm}
\usepackage{multirow}

\newtheorem{theorem}{Theorem}

\newtheorem{proposition}[theorem]{\bf{Proposition}}
\newtheorem{lemma}[theorem]{\bf{Lemma}}
\newtheorem{corollary}[theorem]{\bf{Corollary}}

\newcommand{\vect}[1]{\mathbf{#1}}
\newcommand{\set}[1]{\mathbb{#1}}
\newcommand{\sett}[1]{\mathcal{#1}}

\newcommand{\vectsymb}[1]{\boldsymbol{#1}}

\DeclareMathOperator*{\argmax}{argmax}

\DeclareMathOperator*{\sgn}{sgn}

\title{Online Multiple Kernel Learning for Structured Prediction}

\author{Andr\'{e} F. T. Martins$^\ast$$^\dagger$ \qquad Noah A. Smith$^\ast$ \qquad Eric P. Xing$^\ast$\\[2pt]
Pedro M. Q. Aguiar$^\ddagger$ \qquad M\'{a}rio A. T. Figueiredo$^\dagger$ \\[2pt]
{\tt \{afm,nasmith,epxing\}@cs.cmu.edu}\\
{\tt aguiar@isr.ist.utl.pt}, \qquad {\tt mtf@lx.it.pt} \\[5pt]
$^\ast$School of Computer Science \\
Carnegie Mellon University, Pittsburgh, PA, USA\\[5pt]
$^\dagger$Instituto de Telecomunica\c{c}\~{o}es \\
Instituto Superior T\'{e}cnico, Lisboa, Portugal \\[5pt]
$^\ddagger$Instituto de Sistemas e Rob\'{o}tica \\
Instituto Superior T\'{e}cnico, Lisboa, Portugal}

\begin{document}
\maketitle
\begin{abstract}
Despite the recent progress towards efficient multiple kernel learning (MKL),
the structured output case 
remains an open research front. Current approaches involve repeatedly solving
a batch learning problem, which makes them inadequate for large scale 
scenarios. We propose a new  family of \emph{online} proximal
algorithms for MKL (as well as for group-{\sc lasso} and variants thereof),
which overcomes that drawback. We show regret, convergence, 
and generalization bounds for the proposed method. Experiments on handwriting recognition 
and dependency parsing testify for the successfulness of the approach.
\end{abstract}


\section{Introduction}

Structured prediction \citep{Lafferty2001,Taskar2003,Tsochantaridis2004}
deals with problems with a strong interdependence among the output variables, often with sequential,
graphical, or combinatorial structure.
Despite recent advances toward a unified formalism, obtaining a good
predictor often requires a significant effort in designing kernels ({\it i.e.},
features and similarity measures) and tuning hyperparameters. The slowness in training structured predictors
in large scale settings makes this an expensive process.

The need for careful kernel engineering can be sidestepped using the
kernel learning approach initiated in \cite{Bach2004,Lanckriet2004},
where a combination of multiple kernels is learned from the data.
While multi-class and scalable multiple kernel learning (MKL)
algorithms have been proposed \citep{Sonnenburg2006JMLR,Zien2007,Rakotomamonjy2008,Chapelle2008b,Xu2009,Suzuki2009},
none are well suited for large-scale structured prediction, for the following
reason: all involve an inner loop in which a standard
learning problem ({\it e.g.}, an SVM) is repeatedly solved;
in large-scale structured prediction, it is often prohibitive
to tackle this problem in its batch form, and one typically
resorts to online methods \citep{Bottou1991,Collins2002,Ratliff2006,Collins2008}. These methods are
fast in achieving low generalization error, but converge slowly to the training objective,
thus are unattractive for repeated use in the inner loop.

In this paper, we overcome the above difficulty by proposing a stand-alone online MKL algorithm.
The algorithm is based on the
kernelization of the recent forward-backward splitting scheme {\sc Fobos} \cite{Duchi2009JMLR}
and iterates between subgradient and proximal steps.
In passing,
we improve the {\sc Fobos} regret bound and
show how to efficiently compute the proximal projections associated with the
\emph{squared} $\ell_1$-norm, despite the fact that the underlying optimization problem is not separable.

After reviewing structured prediction and MKL (\S\ref{sec:strucpredmkl}),
we present a wide class of online proximal algorithms (\S\ref{sec:onlineproxalg})
which extend {\sc Fobos} by
handling composite regularizers with multiple proximal steps.
These algorithms have convergence guarantees and are applicable in MKL, group-{\sc lasso} \citep{Yuan2006} and
other structural sparsity formalisms, such as
hierarchical {\sc lasso}/MKL \cite{Bach2008,Zhao2008},
group-{\sc lasso} with overlapping groups \cite{Jenatton2009},
sparse group-{\sc lasso} \citep{Friedman2010}, and the
elastic net MKL \citep{Tomioka2010}.
We apply our MKL algorithm to structured prediction (\S\ref{sec:experiments}), using the two following
testbeds: sequence labeling for handwritten text recognition,
and natural language dependency parsing. We show the potential of our
approach by learning combinations of kernels from
tens of thousands of training instances, with
encouraging results in terms of runtimes, accuracy and identifiability.

\section{Structured Prediction, Group Sparsity, and Multiple Kernel Learning}\label{sec:strucpredmkl}

Let $\mathcal{X}$ and $\mathcal{Y}$ be the input and output sets, respectively.
In structured prediction, to each input $x \in \mathcal{X}$ corresponds
a (structured and exponentially large) set $\mathcal{Y}(x) \subseteq \mathcal{Y}$ of
legal outputs; {\it e.g.}, in sequence labeling,
each $x \in \mathcal{X}$ is an observed sequence and each
$y \in \mathcal{Y}(x)$ is the corresponding sequence of labels;
in parsing, each $x \in \mathcal{X}$ is a string, and each
$y \in \mathcal{Y}(x)$ is a parse tree that spans that string.

Let $\sett{U} \triangleq \{(x,y) \,|\,\, x\! \in\! \sett{X}, y
\! \in\! \mathcal{Y}(x)\}$ be
the set of all legal input-output pairs. Given a labeled
dataset $\mathcal{D}\triangleq\{(x_1,y_{1}),\ldots,(x_m,y_{m})\} \subseteq \sett{U}$,
we want to learn a predictor $h:\sett{X} \rightarrow \sett{Y}$ of the form
\begin{equation}\label{eq:inference}
h(x) \triangleq \arg\max_{y\in \sett{Y}(x)} f(x,y),
\end{equation}
where $f\!:\! \sett{U}\! \rightarrow\! \set{R}$ is a compatibility function.
Problem~\eqref{eq:inference} is called \emph{inference} (or \emph{decoding})
and involves combinatorial optimization ({\it e.g.}, dynamic programming).
In this paper, we use linear functions, $f(x,y) = \langle \vectsymb{\theta}, \vectsymb{\phi}(x,y)\rangle$, where
$\vectsymb{\theta}$ is a parameter vector and $\vectsymb{\phi}(x,y)$ a feature vector.
The structure of the output is usually taken care of by assuming a
decomposition of the form $\vectsymb{\phi}(x,y) = \sum_{r \in \sett{R}} \vectsymb{\phi}_r(x,y_r)$,
where $\sett{R}$ is a set of \emph{parts} and the $y_r$ are partial output assignments
(see \citep{Taskar2003} for details).
Instead of explicit features, one may use a positive definite kernel,
$K : \sett{U}\times \sett{U} \rightarrow \set{R}$, and
let $f$ belong to the induced RKHS $\sett{H}_K$.
Given a convex loss function $L:\sett{H}_K\times \mathcal{X}\times \mathcal{Y}\rightarrow \set{R}$,
the \emph{learning} problem is usually formulated as a minimization of the regularized empirical risk:
\begin{equation}\label{eq:standardlearn}
\min_{f \in \sett{H}_K} \frac{\lambda}{2}\|f\|^2_{\sett{H}_K} + \frac{1}{m}\sum_{i=1}^m L(f;x_i,y_i) ,
\end{equation}
where $\lambda \ge 0$ is a regularization parameter and $\|.\|_{\sett{H}_K}$ is the norm in
$\sett{H}_K$. In structured prediction, the logistic loss (in CRFs) and the structured hinge
loss (in SVMs) are common choices:
\begin{eqnarray}
L_{\mathrm{CRF}}(f; x, y) &\triangleq& \textstyle \log \sum_{y' \in \sett{Y}(x)} \exp
(f(x,y') - f(x,y)),\\
L_{\mathrm{SVM}}(f; x, y) &\triangleq& \textstyle \max_{y' \in \sett{Y}(x)}
f(x,y')-f(x,y) + \ell(y',y).\label{eq:LSVM}
\end{eqnarray}
In \eqref{eq:LSVM},  $\ell :\mathcal{Y} \times \mathcal{Y} \rightarrow \set{R}_+$ is a
user-given cost function.  The solution of \eqref{eq:standardlearn}
can be expressed as a kernel expansion (structured version of the representer theorem
\citep[Corollary 13]{Hofmann2008}).

In the kernel learning framework  \cite{Bach2004,Lanckriet2004},
the kernel is expressed as a convex combination of elements of a finite
set $\{K_1,\ldots,K_p\}$, the coefficients of which are learned from data. That is,
$K\in \sett{K}$, where
\begin{equation}
\sett{K} \triangleq \Bigl\{K = {\textstyle \sum_{j=1}^p } \beta_j K_j\,\,\Big|\,\,\vectsymb{\beta} \in \Delta^p \Bigr\},
\quad \text{with}\quad \Delta^p \triangleq \Bigl\{\vectsymb{\beta} \in \set{R}_+^p \,\, |\,\, {\textstyle \sum_{j=1}^p } \beta_j = 1\Bigr\}.
\end{equation}
The so-called MKL problem is the minimization of \eqref{eq:standardlearn} with respect to $K$.
Letting $\sett{H}_{\sett{K}} = \bigoplus_{j=1}^p \sett{H}_{K_j}$ be the direct sum of the RKHS, this optimization can be written (as
shown in \citep{Bach2004,Rakotomamonjy2008}) as:
\begin{eqnarray}\label{eq:learnkernel1}
f^* = \arg \min_{f \in \sett{H}_{\sett{K}}} \hspace{0.3cm} \frac{\lambda}{2}\biggl( \sum_{j=1}^p \|f_j\|_{\sett{H}_{K_j}}\biggr)^2 + \frac{1}{m}\sum_{i=1}^m L\biggl(\sum_{j=1}^p f_j;x_i,y_i\biggr),
\end{eqnarray}
where the optimal kernel coefficients are $\beta_j^* = \|f^*_j\|_{\sett{H}_{K_j}}/\sum_{l=1}^p \|f^*_l\|_{\sett{H}_{K_l}}$.
For explicit features,  the parameter vector is split into $p$ groups,
$\vectsymb{\theta} = (\vectsymb{\theta}_1,\ldots, \vectsymb{\theta}_p)$,
  and the minimization in \eqref{eq:learnkernel1} becomes
\begin{eqnarray}\label{eq:learnkernel2}
\vectsymb{\theta}^*  = \arg\min_{\vectsymb{\theta} \in \set{R}^d} \frac{\lambda}{2}\|\vectsymb{\theta}\|_{2,1}^2 +
\frac{1}{m}\sum_{i=1}^m L(\vectsymb{\theta};x_i,y_i),
\end{eqnarray}
where $\|\vectsymb{\theta}\|_{2,1} \triangleq \sum_{j=1}^p \|\vectsymb{\theta}_j\|$ is a sum of $\ell_2$-norms,
called the \emph{mixed $\ell_{2,1}$-norm}. The group-{\sc lasso}
criterion \citep{Yuan2006} is similar to \eqref{eq:learnkernel2}, without the
square in the regularization term, revealing a close relationship with MKL \citep{Bach2008JMLR}.
In fact, the two problems are equivalent up to a change of $\lambda$.
The $\ell_{2,1}$-norm regularizer favors \emph{group sparsity}:
groups that are found irrelevant tend to be entirely discarded.

%
%


Early approaches to MKL \citep{Lanckriet2004,Bach2004} considered the dual of \eqref{eq:learnkernel1}
in a QCQP or SOCP form, thus were limited to small scale problems.
Subsequent work focused on scalability: in \citep{Sonnenburg2006JMLR}, a semi-infinite LP formulation
and a cutting plane algorithm are proposed;
SimpleMKL \citep{Rakotomamonjy2008} alternates between learning an SVM and a gradient-based
(or Newton \cite{Chapelle2008b}) update of the kernel weights; other techniques include the
extended level method  \citep{Xu2009} and SpicyMKL \citep{Suzuki2009}, based on
an augmented Lagrangian method.
These are all batch algorithms, requiring the repeated solution of
problems of the form \eqref{eq:standardlearn};
even if one can take advantage of warm-starts,
the
convergence proofs of these methods, when available, rely on the exactness (or prescribed
accuracy in the dual) of these solutions.

In contrast, we tackle \eqref{eq:learnkernel1} and \eqref{eq:learnkernel2}
 in \emph{primal} form.
Rather than repeatedly calling off-the-shelf solvers for \eqref{eq:standardlearn},
we propose a stand-alone online algorithm with runtime comparable
to that of solving a \emph{single} instance of \eqref{eq:standardlearn} by
online methods (the fastest in
large-scale settings \citep{ShalevShwartz2007ICML,Bottou1991}).
This paradigm shift paves the way for extending MKL to structured
prediction, a large territory yet to be explored.

\section{Online Proximal Algorithms}\label{sec:onlineproxalg}

We frame our online MKL algorithm in a wider class of
\emph{online proximal algorithms}.
The theory of proximity operators \citep{Moreau1962}, which is widely known in optimization
and has recently gained prominence in the signal processing community \citep{Combettes2006,Wright2009},
provides tools for analyzing these algorithms and generalizes many known results,
sometimes with remarkable simplicity.
We thus start by summarizing its important concepts in \S\ref{subsec:moreau}, together
with a quick review of convex analysis.


\subsection{Convex Functions, Subdifferentials, Proximity Operators, and Moreau Projections}\label{subsec:moreau}

Throughout, we let $\varphi:\set{R}^p \rightarrow \bar{\set{R}}$ (where $\bar{\set{R}} \triangleq \set{R} \cup \{+\infty\}$)
be a convex,
lower semicontinuous (lsc) (the epigraph $\mathrm{epi}\varphi \triangleq \{(x, t)\in
\set{R}^p\times \set{R} \,|\,\varphi(x) \le t\}$
is closed in $\set{R}^p\!\times\! \set{R}$), and proper ($\exists{\vect{x}}:\varphi(\vect{x}) \!\neq\! +\infty$)
function.
The \emph{subdifferential} of $\varphi$ at
$\vect{x}_0$ is the set
$$\partial \varphi(\vect{x}_0) \triangleq \{\vect{g} \in \set{R}^d \,\,|\,\,\forall \vect{x} \in \set{R}^d, \,\, \varphi(\vect{x})-\varphi(\vect{x}_0) \ge \vect{g}^{\top} (\vect{x}-\vect{x}_0) \},$$
the elements of which are the \emph{subgradients}.
We say that $\varphi$ is \emph{$G$-Lipschitz} in
$\sett{S}\! \subseteq \! \set{R}^d$ if  $\forall\vect{x}\in\sett{S}, \forall\vect{g} \in \partial \varphi(\vect{x}),
 \|\vect{g}\| \le G$. We say that $\varphi$ is \emph{$\sigma$-strongly convex} in $\sett{S}$ if
\begin{equation*}
\forall\vect{x}_0\in\sett{S}, \quad \forall\vect{g} \in \partial \varphi(\vect{x}_0), \quad \forall\vect{x} \in \set{R}^d, \quad \varphi(\vect{x}) \ge \varphi(\vect{x}_0) + \vect{g}^\top (\vect{x} - \vect{x}_0) + (\sigma/2) \|\vect{x} - \vect{x}_0\|^2.
\end{equation*}
The \emph{Fenchel conjugate} of $\varphi$ is
$\varphi^\star\! :\set{R}^p\! \rightarrow \bar{\set{R}}$,
$\varphi^\star(\vect{y}) \triangleq \sup_{\vect{x}} \vect{y}^\top \vect{x}\! - \varphi(\vect{x})$.
Let: $$M_{\varphi}(\vect{y}) \triangleq \inf_{\vect{x}} \frac{1}{2}\|\vect{x} - \vect{y}\|^2 + \varphi(\vect{x}), \quad \text{and}
\quad
\mathrm{prox}_{\varphi}(\vect{y}) = \arg\inf_{\vect{x}}
\frac{1}{2}\|\vect{x} - \vect{y}\|^2 + \varphi(\vect{x});$$
the function $M_{\varphi}:\set{R}^p \! \rightarrow \bar{\set{R}}$ is called
the \emph{Moreau envelope} of $\varphi$, and the map $\mathrm{prox}_{\varphi}\!: \set{R}^p\! \rightarrow \set{R}^p$
is the \emph{proximity operator} of $\varphi$ \citep{Combettes2006,Moreau1962}.
Proximity operators generalize Euclidean projectors: consider the
case  $\varphi = \iota_{\mathcal{C}}$, where $\mathcal{C} \subseteq \set{R}^p$ is a convex set
and $\iota_{\mathcal{C}}$ denotes its indicator ({\it i.e.},
$\varphi(\vect{x}) = 0$ if $\vect{x} \in \mathcal{C}$ and $+\infty$ otherwise).
Then, $\mathrm{prox}_{\varphi}$ is the Euclidean projector onto $\mathcal{C}$ and
$M_{\varphi}$ is the residual.
Two other important examples of proximity operators follow:
\begin{itemize}
\item if
$\varphi(\vect{x}) = (\lambda/2) \|\vect{x}\|^2$, then
$\mathrm{prox}_{\varphi}(\vect{y}) = \vect{y}/(1+\lambda)$; 
\item if $\varphi(\vect{x}) = \tau \|\vect{x}\|_1$,
then $\mathrm{prox}_{\varphi}(\vect{y}) = \mathrm{soft}(\vect{y}, \tau)$ is
the \emph{soft-threshold} function \cite{Wright2009},
defined as $[\mathrm{soft}(\vect{y}, \tau)]_k = \sgn(y_k) \cdot \max\{0,|y_k|-\tau\}$.
\end{itemize}

If $\varphi:\set{R}^{d_1}\times\ldots \times \set{R}^{d_p}\rightarrow \bar{\set{R}}$ is
(group-)separable, {\it i.e.},  $\varphi(\vect{x}) = \sum_{k=1}^p \varphi_k(\vect{x}_k)$,
where $\vect{x}_k \in \set{R}^{d_k}$,
then its proximity operator inherits the same (group-)separability: $[\mathrm{prox}_{\varphi}(\vect{x})]_k = \mathrm{prox}_{\varphi_k}(\vect{x}_k)$ \cite{Wright2009}. For example, the proximity
operator of the mixed $\ell_{2,1}$-norm, which is group-separable, has this form.
The following proposition, that we prove in Appendix~\ref{sec:proof_prop_moreaugroupreduct}, extends this result by
showing how to compute proximity operators of functions (maybe not separable)
that only depend on the $\ell_2$-norms of groups of components; {\it e.g.},
the proximity operator of the squared $\ell_{2,1}$-norm  reduces to
that of squared $\ell_1$.

\begin{proposition}\label{prop:moreaugroupreduct}
Let $\varphi:\set{R}^{d_1}\times\ldots \times \set{R}^{d_p}\rightarrow \bar{\set{R}}$ be of the form
$\varphi(\vect{x}_1,\ldots,\vect{x}_p) = \psi(\|\vect{x}_1\|,\ldots,\|\vect{x}_p\|)$ for some
$\psi:\set{R}^p \rightarrow \bar{\set{R}}$. Then,
$M_{\varphi}(\vect{x}_1,\ldots,\vect{x}_p) = M_{\psi}(\|\vect{x}_1\|,\ldots,\|\vect{x}_p\|)$
and
$[\mathrm{prox}_{\varphi}(\vect{x}_1,\ldots,\vect{x}_p)]_k = [\mathrm{prox}_{\psi}(\|\vect{x}_1\|,\ldots,\|\vect{x}_p\|)]_k (\vect{x}_k/\|\vect{x}_k\|)$.
\end{proposition}

Finally, we recall the \emph{Moreau decomposition},
relating the proximity operators of Fenchel conjugate functions
\citep{Combettes2006} and present a corollary (proved in Appendix~\ref{sec:proof_cor_bound2})
that is the key to our regret bound in \S\ref{subsec:regret}.
\begin{proposition}[\citet{Moreau1962}]\label{prop:moreau}
For any convex, lsc, proper function $\varphi:\set{R}^p \rightarrow \bar{\set{R}}$,
\begin{equation}
\vect{x} = \mathrm{prox}_{\varphi}(\vect{x}) + \mathrm{prox}_{\varphi^{\star}}(\vect{x})
\quad \mbox{and} \quad
\|\vect{x}\|^2/2 = M_{\varphi}(\vect{x}) + M_{\varphi^{\star}}(\vect{x}).\label{eq:moreau}
\end{equation}
\end{proposition}
\begin{corollary}\label{cor:bound2}
Let $\varphi:\set{R}^p \rightarrow \bar{\set{R}}$ be as in Prop.~\ref{prop:moreau},
and  $\bar{\vect{x}} \triangleq \mathrm{prox}_{\varphi}(\vect{x})$.
Then, any $\vect{y} \in \set{R}^p$ satisfies
\begin{eqnarray}\label{eq:bound2}
\|\vect{y} - \bar{\vect{x}}\|^2 - \|\vect{y} - \vect{x}\|^2 \le 2(\varphi(\vect{y}) - \varphi(\bar{\vect{x}})).
\end{eqnarray}
\end{corollary}
Although the Fenchel dual $\varphi^{\star}$ does not show up in \eqref{eq:bound2},
it has a crucial role in proving Corollary~\ref{cor:bound2}.

\subsection{A General Online Proximal Algorithm for Composite Regularizers}\label{subsec:algo}
The general algorithmic structure that we propose and analyze in this paper,
presented as Alg.~\ref{alg:genproxalg}, deals (in an online\footnote{For simplicity, we focus
on the pure online setting, {\it i.e.}, each parameter update uses a
single observation; analogous algorithms may be derived for the batch and mini-batch cases.}
fashion) with problems of the form
\begin{equation}\label{eq:learningproblem}
\min_{\vectsymb{\theta}\in\Theta} \lambda R(\vectsymb{\theta}) + \frac{1}{m}\sum_{i=1}^m L(\vectsymb{\theta}; x_i,y_i),
\end{equation}
where $\Theta \subseteq \set{R}^d$ is convex%
\footnote{We are particularly interested in the case where $\vectsymb{\theta}\in\Theta$ is a ``vacuous'' constraint whose
goal is to confine each iterate $\vectsymb{\theta}_t$ to a region containing
the optimum, by virtue of the projection step in line~\ref{algline:projectionstep}.
The analysis in \S\ref{subsec:regret} will make this more clear.
The same trick is used in {\sc Pegasos} \citep{ShalevShwartz2007ICML}.} %
and the regularizer $R$ has a composite form $R(\vectsymb{\theta})=\sum_{j=1}^J R_j(\vectsymb{\theta})$.
Like stochastic
gradient descent (SGD \citep{Bottou1991}), Alg.~\ref{alg:genproxalg}
is suitable for problems with large $m$;
it also performs (sub-)gradient steps at each round (line \ref{algline:gradientstep}),
but only w.r.t. the loss function $L$. Obtaining a subgradient
typically involves inference using the current model; {\it e.g.}, loss-augmented inference,
if $L=L_{\mathrm{SVM}}$, or marginal inference if $L=L_{\mathrm{CRF}}$. 
Our algorithm differs from SGD by the inclusion of
$J$ proximal steps w.r.t. to
each term $R_j$ (line \ref{algline:proximalstep}).
As noted in \citep{Duchi2009JMLR,Langford2009},
this strategy is more effective than standard SGD
for sparsity-inducing regularizers, due to their
usual non-differentiability at the zeros,
which causes oscillation and prevents SGD from returning sparse solutions.

When $J=1$, Alg.~\ref{alg:genproxalg} reduces to {\sc Fobos}
\citep{Duchi2009JMLR}, which we kernelize and apply to MKL in \S\ref{subsec:fobosmkl}.
The case $J > 1$ has applications in variants of MKL or group-{\sc lasso} with composite regularizers
\citep{Tomioka2010,Friedman2010,Bach2008,Zhao2008}.
In those cases, the proximity operators of $R_1,\ldots,R_J$ are more easily computed
than that of their sum $R$, making Alg.~\ref{alg:genproxalg} more suitable than {\sc Fobos}. %
We present a few particular instances (all with $\Theta=\set{R}^d$).
\begin{algorithm}[t]
   \caption{Online Proximal Algorithm \label{alg:genproxalg}}
\begin{algorithmic}[1]
   \STATE {\bfseries input:} dataset $\mathcal{D}$, parameter $\lambda$, number of rounds $T$,
   learning rate sequence $(\eta_t)_{t = 1,\ldots,T}$
   \STATE initialize $\vectsymb{\theta}_1 = \vect{0}$; set $m = |\mathcal{D}|$
	\FOR{$t=1$ {\bfseries to} $T$}
	\STATE take a training pair $(x_t, y_t)$ and obtain a subgradient $\vect{g} \in \partial L(\boldsymbol{\theta}_t; x_t, y_t)$\label{algline:gradientstep}
	\STATE $\tilde{\vectsymb{\theta}}_{t} = \vectsymb{\theta}_t - \eta_t \vect{g}$ (gradient step)
	\FOR{$j=1$ {\bfseries to} $J$}
	\STATE $\tilde{\vectsymb{\theta}}_{t+j/J} = \mathrm{prox}_{\eta_t \lambda R_j}(\tilde{\vectsymb{\theta}}_{t+(j-1)/J})$ (proximal step) \label{algline:proximalstep}
	\ENDFOR	
	\STATE ${\vectsymb{\theta}}_{t+1} = \Pi_{\Theta}(\tilde{\vectsymb{\theta}}_{t+1})$ (projection step) \label{algline:projectionstep}
	\ENDFOR
   \STATE \textbf{output:} the last model $\vectsymb{\theta}_{T+1}$ or
   the averaged model $\bar{\vectsymb{\theta}} = \frac{1}{T}\sum_{t=1}^T \vectsymb{\theta}_t$
\end{algorithmic}
\end{algorithm}

\paragraph{Projected subgradient with groups.}
Let $J=1$ and $R$ be the indicator of a convex set $\Theta' \subseteq \set{R}^d$.
Then (see \S\ref{subsec:moreau}), each proximal step is the
Euclidean projection onto $\Theta'$ and Alg.~\ref{alg:genproxalg}
becomes the online projected subgradient algorithm from \citep{Zinkevich2003}.
Letting $\Theta' \triangleq \{\vectsymb{\theta} \in \set{R} \,\,|\,\, \|\vectsymb{\theta}\|_{2,1} \le \gamma\}$
yields an equivalent problem to group-{\sc lasso} and MKL \eqref{eq:learnkernel2}.
Using Prop.~\ref{prop:moreaugroupreduct}, each proximal step reduces to
a projection onto a $\ell_1$-ball
whose dimension is the number of groups (see a fast algorithm in \citep{Duchi2008}). 

\paragraph{Truncated subgradient with groups.}
Let $J=1$ and $R(\vectsymb{\theta}) = \|\vectsymb{\theta}\|_{2,1}$,
so that \eqref{eq:learningproblem} becomes the usual formulation of group-{\sc lasso},
for a general loss $L$.
Then, Alg.~\ref{alg:genproxalg} becomes a group version of truncated gradient
descent \citep{Langford2009}, studied in \citep{Duchi2009JMLR} for multi-task learning.
Similar batch algorithms have also been proposed \citep{Wright2009}.
The reduction from $\ell_{2,1}$ to $\ell_1$ can again be made due to
Prop.~\ref{prop:moreaugroupreduct}; and each proximal step becomes a
simple soft thresholding operation (as shown in \S\ref{subsec:moreau}).

\paragraph{Proximal subgradient for \emph{squared} mixed $\ell_{2,1}$.}
With $R(\vectsymb{\theta}) = \frac{1}{2}\|\vectsymb{\theta}\|_{2,1}^2$,
we have the MKL problem \eqref{eq:learnkernel2}. Prop.~\ref{prop:moreaugroupreduct}
allows reducing each proximal step w.r.t.~the squared $\ell_{2,1}$ to one
w.r.t.~the squared $\ell_{1}$; however, unlike in the previous example,
squared $\ell_{1}$ is not separable. This apparent difficulty has led
some authors ({\it e.g.}, \cite{Suzuki2009}) to remove the square
from $R$, which yields the previous example.  However, despite the
non-separability of $R$, the proximal steps can still be
efficiently computed: see Alg.~\ref{alg:softproj}. This algorithm requires
sorting the weights of each group, which has $O(p \log p)$ cost; 
we show its correctness in Appendix~\ref{sec:proximoper_sql1}. 
Non-MKL applications of the squared
$\ell_{2,1}$ norm are found in \citep{Kowalski2009,Zhou2010AISTATS}.

\begin{algorithm}[t]
   \caption{Moreau Projection for $\ell_1^2$ \label{alg:softproj}}
\begin{algorithmic}[1]
   \STATE {\bfseries input:}  vector $\vect{x} \in \set{R}^d$ and parameter $\lambda > 0$
   \STATE sort the entries of $|\vect{x}|$ into $\vect{y}$ ({\it i.e.}, such that $y_1 \ge \ldots \ge y_p$)
   \STATE find $\rho = \max \left\{j \in \{1,\ldots,p\} \,\,|\,\, y_{j} - (\lambda/(1 + j\lambda)) \sum_{r=1}^j y_{r} > 0\right\}$
   \STATE {\bfseries output:} $\vect{z} = \mathrm{soft}(\vect{x},\tau)$, where $\tau = (\lambda/(1 + \rho\lambda)) \sum_{r=1}^{\rho} y_{r}$ 
\end{algorithmic}
\end{algorithm}

\paragraph{Other variants of group-{\sc lasso} and MKL.}
In hierarchical {\sc lasso} and group-{\sc lasso} with overlaps \citep{Bach2008,Zhao2008,Jenatton2009},
each feature may appear in more than one group. Alg.~\ref{alg:genproxalg} handles these problems
by enabling a proximal step for each group. Sparse group-{\sc lasso} \citep{Friedman2010} simultaneously
promotes group-sparsity and sparsity \emph{within} each group, by using $R(\vectsymb{\theta}) = \sigma \|\vectsymb{\theta}\|_{2,1} + (1 - \sigma) \|\vectsymb{\theta}\|_1$; Alg.~\ref{alg:genproxalg} can
handle this regularizer by using two proximal steps, both involving simple soft-thresholding: one
at the group level, and another within each group.
In non-sparse MKL (\citep{Kloft2010}, \S4.4), $R=\frac{1}{2}\sum_{k=1}^p \|\vectsymb{\theta}_k\|^q$.
Invoking Prop.~\ref{prop:moreaugroupreduct} and separability, the resulting proximal step
amounts to solving $p$ scalar equations of the form $x - x_0 + \lambda \, \eta_t\, q\, x^{q-1} = 0$,
also valid for $q \ge 2$ (unlike the method described in \citep{Kloft2010}).

\subsection{Regret, Convergence, and Generalization Bounds}\label{subsec:regret}
We next show that, for a convex loss $L$ and under standard assumptions,
Alg.~\ref{alg:genproxalg} converges
up to $\epsilon$ precision, with high confidence, in $O(1/\epsilon^2)$ iterations.
If $L$ or $R$ are strongly convex,
this bound is improved to  $\tilde{O}(1/\epsilon)$, where $\tilde{O}$ hides logarithmic terms.
Our proofs combine tools of online convex programming \citep{Zinkevich2003,Hazan2007} and
classical results about proximity operators \citep{Moreau1962,Combettes2006}. The key is the following lemma
(that we prove in Appendix~\ref{sec:proof_lemma_softprojgrad2}).

\begin{lemma}\label{lemma:softprojgrad2}
Assume that  $\forall (x,y) \in \sett{U}$, the loss $L(\cdot;x,y)$ is convex
and $G$-Lipschitz on $\Theta$,
and that the regularizer $R = R_1 + \ldots + R_J$ satisfies the following conditions:
{\bf (i)} each $R_j$ is convex;
{\bf (ii)} $\forall \vectsymb{\theta} \in \Theta, \,
\forall j'<j , \,  R_{j'}(\vectsymb{\theta}) \ge R_{j'}(\mathrm{prox}_{\lambda R_j}(\vectsymb{\theta}))$
(each proximity operator $\mathrm{prox}_{\lambda R_j}$ does not increase the previous
$R_{j'}$);
{\bf (iii) }  $R(\vectsymb{\theta}) \ge R(\Pi_{\Theta}(\vectsymb{\theta}))$
(projecting the argument onto $\Theta$ does not increase $R$). Then, for any
$\bar{\vectsymb{\theta}} \in \Theta$, at each round $t$ of Alg.~\ref{alg:genproxalg},
\begin{eqnarray}\label{eq:softprojgradbound2}
L(\vectsymb{\theta}_t) + \lambda R(\vectsymb{\theta}_{t+1}) \le
L(\bar{\vectsymb{\theta}}) + \lambda  R(\bar{\vectsymb{\theta}}) +
\frac{\eta_t }{2} G^2 + \frac{\|\bar{\vectsymb{\theta}} - \vectsymb{\theta}_t\|^2 - \|\bar{\vectsymb{\theta}} - \vectsymb{\theta}_{t+1}\|^2}{2\eta_t }.
\end{eqnarray}
If, in addition, $L$ is $\sigma$-strongly convex, then the bound in \eqref{eq:softprojgradbound2} can be strengthened to
\begin{eqnarray}\label{eq:softprojgradboundstronglyconvex}
L(\vectsymb{\theta}_t) + \lambda R(\vectsymb{\theta}_{t+1}) \le
L(\bar{\vectsymb{\theta}}) + \lambda  R(\bar{\vectsymb{\theta}}) +
\frac{\eta_t }{2} G^2 + \frac{\|\bar{\vectsymb{\theta}} - \vectsymb{\theta}_t\|^2 - \|\bar{\vectsymb{\theta}} - \vectsymb{\theta}_{t+1}\|^2}{2\eta_t}
-\frac{\sigma}{2}\|\bar{\vectsymb{\theta}} - \vectsymb{\theta}_t\|^2.
\end{eqnarray}
\end{lemma}

\smallskip

A related, but less tight, bound for $J=1$ was derived in \cite{Duchi2009JMLR};
instead of our term $\frac{\eta}{2} G^2$ in \eqref{eq:softprojgradbound2},
the bound of \citep{Duchi2009JMLR} has $7 \frac{\eta}{2} G^2$.%
\footnote{This can be seen from their Eq.~9, setting $A = 0$ and $\eta_t=\eta_{t+\frac{1}{2}}$.} %
When $R = \|\cdot\|_1$, {\sc Fobos} becomes the truncated gradient algorithm of \cite{Langford2009}
and our bound matches the one therein derived, closing the gap between \citep{Duchi2009JMLR} and \citep{Langford2009}.
The classical result in Prop.~\ref{prop:moreau},
relating Moreau projections and Fenchel duality, is the crux of our bound, via Corollary~\ref{cor:bound2}.
Finally, note that the conditions {\bf (i)}--{\bf (iii)} are not restrictive: they hold whenever the
proximity operators are shrinkage functions  ({\it e.g.}, if $R_j = \|\vectsymb{\theta}\|_{p_j}^{q_j}$, with $p_j,q_j \ge 1$).

We next characterize Alg.~\ref{alg:genproxalg} in terms of its
cumulative regret w.r.t.~the best fixed hypothesis, {\it i.e.},
\begin{equation}
\mathrm{Reg}_T \triangleq \sum_{t=1}^T \left( \lambda R(\vectsymb{\theta}_t) + L(\vectsymb{\theta}_t; x_t,y_t) \right) -
\min_{\vectsymb{\theta} \in \Theta} \sum_{t=1}^T \left( \lambda R(\vectsymb{\theta}) + L(\vectsymb{\theta}; x_t,y_t)\right).
\end{equation}

\begin{proposition}[regret bounds with fixed and decaying learning rates]\label{prop:softprojgrad2}
Assume the conditions of Lemma~\ref{lemma:softprojgrad2}, along with $R \ge 0$ and $R(\vect{0})=0$. Then:
\newcounter{sublist}
\begin{list}
{\arabic{sublist}.}{\usecounter{sublist}\setlength{\leftmargin}{0.25cm}\setlength{\itemsep}{0cm}\setlength{\topsep}{0cm}}
\item Running Alg.~\ref{alg:genproxalg} with fixed learning rate $\eta$ yields
\begin{equation}
\mathrm{Reg}_T \le \frac{\eta T}{2} G^2 + \frac{\|\vectsymb{\theta}^*\|^2}{2\eta}, \quad
\text{where $\vectsymb{\theta}^* = \arg \min_{\vectsymb{\theta} \in \Theta} \sum_{t=1}^T \left( \lambda R(\vectsymb{\theta}) + L(\vectsymb{\theta}; x_t,y_t)\right)$}.
\end{equation}
Setting $\eta = \|\vectsymb{\theta}^*\|/(G\sqrt{T})$ yields a sublinear regret of $\|\vectsymb{\theta}^*\|G\sqrt{T}$.
(Note that this requires knowing in advance $\|\vectsymb{\theta}^*\|$ and the number of rounds $T$.)
\item Assume that $\Theta$ is bounded with diameter $F$ ({\it i.e.}, $\forall\vectsymb{\theta},\vectsymb{\theta}' \in \Theta$, $\|\vectsymb{\theta} - \vectsymb{\theta}'\| \le F$).
Let the learning rate be $\eta_t = \eta_0/\sqrt{t}$, with arbitrary $\eta_0 > 0$.
Then,
\begin{eqnarray}\label{eq:regretdecayrate1}
\mathrm{Reg}_T &\le& 
\left(\frac{F^2}{2\eta_0} + G^2 \eta_0\right)\sqrt{T}.
\end{eqnarray}
Optimizing the bound gives $\eta_0 = F/(\sqrt{2}G)$, yielding $\mathrm{Reg}_T \le FG\sqrt{2T}$.
\item If $L$ is $\sigma$-strongly convex, and $\eta_t = 1/(\sigma t)$, we obtain a logarithmic regret bound:
\begin{equation}\label{eq:regretdecayrate2}
\mathrm{Reg}_T \le G^2 (1 + \log T) / (2\sigma).
\end{equation}
\end{list}
\end{proposition}

Similarly to other analyses of online learning algorithms, once an online-to-batch conversion is specified, regret bounds allow us to
obtain PAC bounds on optimization and generalization errors.
The following proposition can be proved using the same techniques as in \citep{Cesabianchi2004,ShalevShwartz2007ICML}.
\begin{proposition}[optimization and estimation error]
If the assumptions of Prop.~\ref{prop:softprojgrad2} hold and $\eta_t = \eta_0/\sqrt{t}$ as in 2., then
the version of Alg.~\ref{alg:genproxalg} that returns the averaged model
solves the optimization problem \eqref{eq:learningproblem} with accuracy $\epsilon$
in $T = O((F^2 G^2 + \log(1/\delta))/\epsilon^2)$ iterations, with
probability at least $1-\delta$.
If $L$ is also $\sigma$-strongly convex and  $\eta_t = 1/(\sigma t)$ as in 3., then, for the version of Alg.~\ref{alg:genproxalg}
that returns $\vectsymb{\theta}_{T+1}$,  we get $T = \tilde{O}(G^2/(\sigma \delta \epsilon))$.
The generalization bounds are of the same orders.
\end{proposition}

We now pause to see how the analysis applies to some concrete cases.
The requirement that the loss is $G$-Lipschitz holds for
the hinge and logistic losses, where $G = 2\max_{u\in\sett{U}} \|\vectsymb{\phi}(u)\|$ (see Appendix~\ref{sec:lipschitz}).
These losses are not strongly convex,
and therefore Alg.~\ref{alg:genproxalg} has only $O(1/\epsilon^2)$ convergence.
If the \emph{regularizer} $R$ is $\sigma$-strongly convex,
a possible workaround to obtain $\tilde{O}(1/\epsilon)$ convergence is to let $L$
``absorb'' that strong convexity by
redefining
$\tilde{L}(\vectsymb{\theta};x_t,y_t) = L(\vectsymb{\theta};x_t,y_t) + \sigma\|\vectsymb{\theta}\|^2/2$.
Since neither the $\ell_{2,1}$-norm nor its square are
strongly convex, we cannot use this trick for the MKL case \eqref{eq:learnkernel2},
but it \emph{does} apply for non-sparse MKL \citep{Kloft2010} ($\ell_{2,q}$-norms are strongly convex for $q > 1$)
and for elastic net MKL \citep{Suzuki2009}.
Still, the $O(1/\epsilon^2)$ rate for MKL is
competitive with the best batch algorithms; {\it e.g.},
the method in \cite{Xu2009} achieves $\epsilon$ primal-dual
gap in $O(1/\epsilon^2)$ iterations. 
Some losses of interest ({\it e.g.}, the squared loss, or the modified loss $\tilde{L}$ above)
are $G$-Lipschitz in any compact subset of
$\set{R}^d$ but not in $\set{R}^d$.
However, if it is known in advance that
the optimal solution must lie in some compact convex set $\Theta$,
we can add a vacuous constraint and run Alg.~\ref{alg:genproxalg} with the projection step,
making the analysis still applicable; we present concrete examples in Appendix~\ref{sec:lipschitz}.

\subsection{Online MKL} \label{subsec:fobosmkl}
The instantiation of Alg.~\ref{alg:genproxalg} for $R(\vectsymb{\theta}) = \frac{1}{2} \|\vectsymb{\theta}\|^2_{2,1}$
 yields Alg.~\ref{alg:onlinemkl}. We consider $L = L_{\mathrm{SVM}}$; adapting
to any generalized linear model ({\it e.g.}, $L = L_{\mathrm{CRF}}$) is straightforward.
As discussed in the last paragraph of \S\ref{subsec:regret}, it may be necessary to consider ``vacuous''
projection steps to ensure fast convergence. Hence, an optional upper bound $\gamma$ on $\|\vectsymb{\theta}\|$
is accepted as input. Suitable values of $\gamma$ for the SVM and CRF case are
given in  Appendix~\ref{sec:lipschitz}. %
In line~\ref{algline:mklscores}, the scores of candidate outputs are computed groupwise;
in structured prediction (see \S\ref{sec:strucpredmkl}), a factorization
over parts is assumed and the scores are for partial output assignments (see \cite{Taskar2003,Tsochantaridis2004} for details).
The key novelty of Alg.~\ref{alg:onlinemkl} is in line~\ref{algline:mklproximalstep}, where
the group structure is taken into account, by applying a proximity operator which
corresponds to a groupwise shrinkage/thresolding, where some groups may be set to zero.

Although Alg.~\ref{alg:onlinemkl} is written in parametric form, 
it can be kernelized, as shown next (one can also use explicit features in some groups,
and implicit in others). Observe that the parameters of the $k$th group after round $t$ can be written as
$\vectsymb{\theta}_k^{t+1} = \sum_{s = 1}^t \alpha_{ks}^{t+1} (\vectsymb{\phi}_k(x_s, y_s) - \vectsymb{\phi}_k(x_s, \hat{y}_s))$, where
\[
\alpha_{ks}^{t+1} = \eta_s \prod_{r = s}^t \left( (b_k^r/\tilde{b}_k^r) \min\{1, \gamma / \|\tilde{\vectsymb{\theta}}^{r+1}\|\} \right)
= \left\{
\begin{array}{ll}
\eta_t (b_k^t/\tilde{b}_k^t) \min\{1, \gamma / \|\tilde{\vectsymb{\theta}}^{t+1}\|\} & \text{if $s = t$}\\
\alpha_{ks}^{t} (b_k^t/\tilde{b}_k^t) \min\{1, \gamma / \|\tilde{\vectsymb{\theta}}^{t+1}\|\} & \text{if $s < t$.}
\end{array}
\right.
\]
Therefore, the inner products  in line~\ref{algline:mklscores} can be kernelized.
The cost of this step is $O(\min\{m,t\})$, instead of the $O(d_k)$ (where $d_k$
is the dimension of the $k$th group) for the explicit feature case.
After the decoding step (line \ref{algline:mkldecode}), the
supporting pair $(x_t,\hat{y}_t)$ is stored.
Lines~\ref{algline:mklproximalstep_0}, \ref{algline:mklprojectionstep} and \ref{algline:mklweights}
require the \emph{norm} of each group, which can be manipulated using kernels:
indeed, after each gradient step (line \ref{algline:mklgradientstep}), we have (denoting $u_t=(x_t,y_t)$ and $\hat{u}_t=(x_t,\hat{y}_t)$):
\begin{eqnarray}
\|\tilde{\vectsymb{\theta}}_k^{t}\| ^2 &=& \|\vectsymb{\theta}^t_k\|^2 -2\eta_t \langle \vectsymb{\theta}^t_k, \vectsymb{\phi}_k(x_t, y_t) \rangle
+ \eta_t^2\|\vectsymb{\phi}_k(x_t, \hat{y}_t) - \vectsymb{\phi}_k(x_t, y_t)\|^2\nonumber\\
&=& \|\vectsymb{\theta}^t_k\|^2 -2\eta_t f_k(\hat{u}_t)
+ \eta_t^2 (K_k(u_t,u_t)+ K_k(\hat{u}_t,\hat{u}_t) - 2K_k(u_t,\hat{u}_t));
\end{eqnarray}
and the proximal and projection steps merely scale these norms.
When the algorithm terminates, it returns the kernel weights $\vectsymb{\beta}$ and the sequence $(\alpha_{kt}^{T+1})$.

\begin{algorithm}[t]
   \caption{Online-{\sc mkl} \label{alg:onlinemkl}}
\begin{algorithmic}[1]
   \STATE {\bfseries input:} $\mathcal{D}$, $\lambda$, $T$, radius $\gamma$,
   learning rate sequence $(\eta_t)_{t = 1,\ldots,T}$
   \STATE initialize $\vectsymb{\theta}^1 \leftarrow \vect{0}$
	\FOR{$t=1$ {\bfseries to} $T$}
	\STATE take an instance $x_t$, $y_t$ and compute scores $f_k(x_t,y_t') = \langle \vectsymb{\theta}^t_k, \vectsymb{\phi}_k(x_t, y_t') \rangle$, for $k=1,\ldots,p$ \label{algline:mklscores}
	\STATE decode: $\hat{y}_t \in \argmax_{y_t' \in \sett{Y}(x)} \sum_{k=1}^p f_k(x_t,y_t') + \ell(y_t',y_t)$ \label{algline:mkldecode}
	\STATE Gradient step: $\tilde{\vectsymb{\theta}}_k^{t} = \vectsymb{\theta}^t_k - \eta_t (\vectsymb{\phi}_k(x_t, \hat{y}_t) - \vectsymb{\phi}_k(x_t, y_t))$ \label{algline:mklgradientstep}
	\STATE compute weights $\tilde{b}_k^t \!=\! \|\tilde{\vectsymb{\theta}}_k^t\|$,  $k\!=\!1,\ldots,p$,
	and shrink them $\vect{b}^t \!=\! \mathrm{prox}_{\eta_t \lambda \|.\|_{2,1}^2}(\vect{\tilde{b}}^t)$ with Alg.~\ref{alg:softproj} \label{algline:mklproximalstep_0}
	\vspace{-.3cm}	
	\STATE Proximal step: $\tilde{\vectsymb{\theta}}^{t+1}_k = (b_k^t / \tilde{b}_k^t) \cdot \tilde{\vectsymb{\theta}}_k^{t}$, for $k=1,\ldots,p$ \label{algline:mklproximalstep}
	\STATE Projection step: ${\vectsymb{\theta}}^{t+1} = \tilde{\vectsymb{\theta}}^{t+1} \cdot \min\{1, \gamma / \|\tilde{\vectsymb{\theta}}^{t+1}\|\}$  \label{algline:mklprojectionstep}
	\ENDFOR
	\STATE compute $\beta_k = \|\vectsymb{\theta}_k^{T+1}\| / \sum_{l=1}^p \|\vectsymb{\theta}_l^{T+1}\|$, for $k=1,\ldots,p$ \label{algline:mklweights}
   \STATE return $\vectsymb{\beta}$, and the last model $\vectsymb{\theta}^{T+1}$
\end{algorithmic}
\end{algorithm}

In case of sparse explicit features, an implementation trick analogous to the one used in
\citep{ShalevShwartz2007ICML} (where each $\vectsymb{\theta}_k$ is
represented by its norm and an unnormalized vector) can substantially reduce
the amount of computation. In the case of implicit features with a sparse
kernel matrix, a sparse storage of this matrix can also  significantly
speed up the algorithm, eliminating its dependency on $m$ in line~\ref{algline:mklscores}.
Note also that all steps involving group-specific computation can be carried out in parallel using
multiple machines, which makes Alg.~\ref{alg:onlinemkl} suitable for combining many kernels (large $p$).

\section{Experiments}\label{sec:experiments}
\paragraph{Handwriting recognition.}
We use the OCR dataset of \cite{Taskar2003} ({\footnotesize{\url{www.cis.upenn.edu/~taskar/ocr}}}),
which has $6877$ words written by $150$ people ($52152$ characters).
Each character is a $16$-by-$8$ binary image, {\it i.e.}, a $128$-dimensional vector (our input)
and has one of $26$ labels ({\tt a-z}; the outputs to predict). Like in \citep{Taskar2003}, we address this
sequence labeling problem with a structured SVM; however, we \emph{learn} the
kernel from the data, via Alg.~\ref{alg:onlinemkl}.
We use an indicator basis function to represent the correlation between
consecutive outputs. Our first experiment (reported in the upper part of
Tab.~\ref{tab:resocr}) compares linear, quadratic, and Gaussian kernels,
either used individually, combined
via a simple average, or with MKL. The results show that  MKL outperforms
the others by $2\%$ or more.

The second experiment aims at showing the ability of Alg.~\ref{alg:onlinemkl} to exploit
both \emph{feature} and \emph{kernel} sparsity by learning a combination of a
\emph{linear} kernel (explicit features) with a \emph{generalized $B_1$-spline} kernel,
given by $K(\vect{x},\vect{x}') = \max\{0, 1 - \|\vect{x}-\vect{x}'\|/h\}$, with
$h$ chosen so that the kernel matrix has $\sim 95\%$ zeros. The rationale is to
combine the strength of a simple feature-based kernel with that of one 
depending only on a few nearest neighbors. The results (Tab.~\ref{tab:resocr},
bottom part) show that the MKL outperforms by $\sim 10\%$ the individual kernels,
and by more than $2\%$ the averaged kernel. Perhaps more importantly, 
the accuracy is not much worse than the best one obtained
in the previous experiment, while the runtime is much faster ($15$ versus $279$ seconds).  

\begin{table}[t]
\centering
{\small
\begin{tabular}{l|cc}
\multicolumn{1}{c|}{Kernel}  &\multicolumn{1}{c}{\rotatebox{0}{Training}} &\multicolumn{1}{c}{\rotatebox{0}{Test Acc.}} \\
\multicolumn{1}{c|}{}  &\multicolumn{1}{c}{\rotatebox{0}{Runtimes}} &\multicolumn{1}{c}{\rotatebox{0}{(per char.)}} \\
\hline
Linear ($L$)         					& 6 sec. & $72.8 \pm 4.4\%$ \\
Quadratic ($Q$)   						& 116 sec. & $85.5 \pm 0.3\%$ \\
Gaussian ($G$) ($\sigma^2=5$)  		& 123 sec. & $84.1 \pm 0.4\%$ \\
Average $(L+Q+G)/3$    		& 118 sec. & $84.3 \pm 0.3\%$ \\
MKL $\beta_1 L+\beta_2 Q+\beta_3 G$	& 279 sec. & $87.5 \pm 0.4\%$ \\
\hline
$B_1$-Spline ($B_1$)			& 8 sec. & $75.4 \pm 0.9\%$\\
Average $(L + B_1)/2$    	& 15 sec.  &  $83.0 \pm 0.3\%$ \\
MKL $\beta_1 L+\beta_2 B_1$ & 15 sec. & $85.2 \pm 0.3\%$ \\
\hline\hline
\end{tabular}}
\caption{Results for handwriting recognition. Averages over $10$ runs on the same folds as in \citep{Taskar2003},
training on one and testing on the others. The linear and quadratic kernels
are normalized to unit diagonal. In all cases, $20$ epochs were used,
with $\eta_0$ in \eqref{eq:regretdecayrate1} picked from
$\{0.01,0.1,1,10\}$ by selecting the one that most decreases the objective after
$5$ epochs. Results are for the best regularization coefficient $C = 1/(\lambda m)$
(chosen from $\{0.1, 1,10,10^{2},10^{3},10^{4}\})$.}\label{tab:resocr}
\end{table}

\paragraph{Dependency parsing.}
We trained non-projective dependency parsers for English,
using the dataset from the CoNLL-2008 shared task \cite{Surdeanu2008}
($39278$ training sentences, $\sim 10^6$ tokens, and $2399$ test sentences).
The output to be predicted from each input sentence is the set of dependency
arcs,  linking \emph{heads} to \emph{modifiers}, that must define a
spanning tree (see example in Fig.~\ref{fig:deptree}).
We use arc-factored models, where the feature vectors decompose as
$\vectsymb{\phi}(x,y) = \sum_{(h,m) \in y} \vectsymb{\phi}_{h,m}(x)$.
Although they are not the state-of-the-art for this task,
exact inference is tractable via minimum spanning tree
algorithms \citep{McDonald2005b}.
We defined $507$ feature templates for each candidate arc by conjoining the words, lemmas, and parts-of-speech of
the head $h$ and the modifier $m$, as well as the parts-of-speech of the surrounding words, and the distance and direction of attachment.
This yields a large scale problem, with $>50$ million features
instantiated. The feature vectors associated with each candidate arc, however, are very sparse and this is exploited in
the implementation. We ran Alg.~\ref{alg:onlinemkl} with explicit features, with each group standing
for a feature template.
MKL did not outperform a standard SVM in this experiment ($90.67\%$ against $90.92\%$); however, it
showed a good performance at pruning irrelevant feature templates
(see Fig.~\ref{fig:deptree}, bottom right). Besides \emph{interpretability}, which may be useful for
the understanding of the syntax of natural languages, this pruning is also appealing in a
two-stage architecture, where a standard learner at a second stage will only need to handle
a small fraction of the templates initially hypothesized.

\begin{figure}[t]
\centering
\begin{tabular}[b]{c}
\includegraphics[width=0.5\columnwidth]{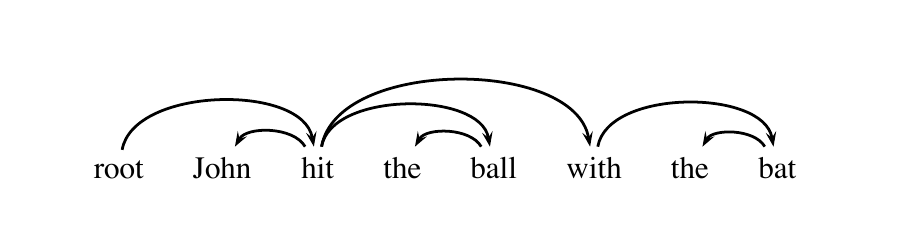}\\
\includegraphics[width=0.27\columnwidth]{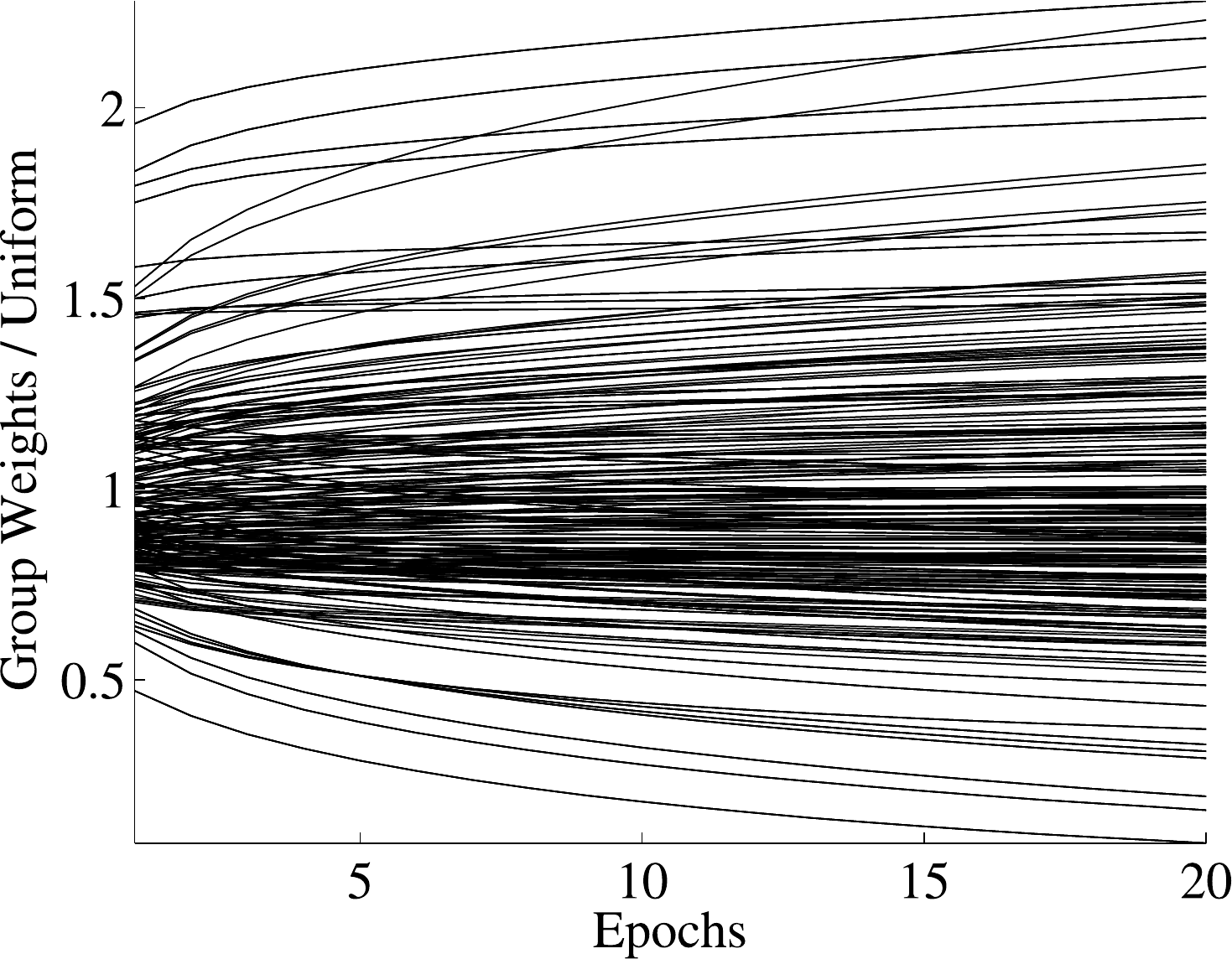} %
\includegraphics[width=0.27\columnwidth]{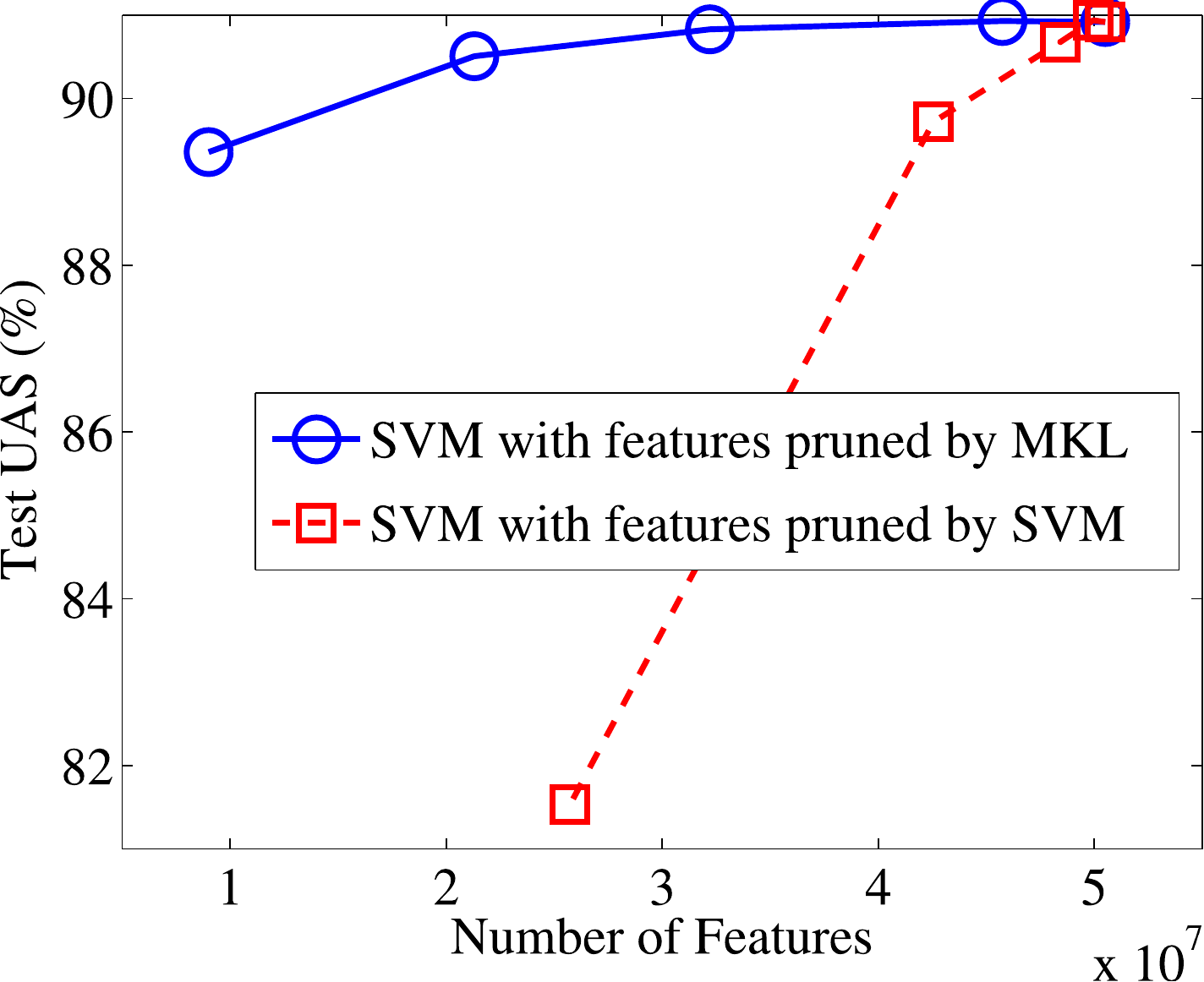}\\
\end{tabular}
\caption{Top: a dependency parse tree (adapted from \citep{McDonald2005b}).
Bottom left: group weights along the epochs of Alg.~\ref{alg:onlinemkl}.
Bottom right: results of standard SVMs trained on sets of feature templates of
sizes $\{107,207,307,407,507\}$, either selected via a standard SVM or by MKL
(the UAS---{\it unlabeled attachment score}---is the
fraction of non-punctuation words whose head was correctly assigned.)}\label{fig:deptree}
\end{figure}

\section{Conclusions}\label{sec:conclusion}
We introduced a
new class of online proximal algorithms that extends {\sc Fobos} and is
applicable to many variants of MKL and group-{\sc lasso}.
We provided regret, convergence, and generalization bounds,
and used the algorithm for learning the kernel
in large-scale structured prediction tasks.

Our work may impact other problems. In structured prediction,
the ability to promote structural sparsity suggests that it is
possible to learn simultaneously the structure and the parameters
of the graphical models. The ability to learn the kernel online
offers a new paradigm for problems in which the underlying
geometry (induced by the similarities between
objects) evolves over time: algorithms that
adapt the kernel while learning are robust to
certain kinds of concept drift. We plan to explore these directions in future work.




\appendix

\section{Proof of Proposition~\ref{prop:moreaugroupreduct}}\label{sec:proof_prop_moreaugroupreduct}

We have respectively:
\begin{eqnarray}
M_{\varphi}(\vect{x}_1,\ldots,\vect{x}_p) &=& \min_{\vect{y}} \frac{1}{2}\|\vect{y} - \vect{x}\|^2 + \varphi(\vect{y})\nonumber\\
&=& \min_{\vect{y}_1,\ldots,\vect{y}_p} \frac{1}{2}\sum_{k=1}^p\|\vect{y}_k - \vect{x}_k\|^2 + \psi(\|\vect{y}_1\|,\ldots,\|\vect{y}_p\|)\nonumber\\
&=& \min_{\vect{u}\in \set{R}^p_+}  \psi(u_1,\ldots,u_p) + \min_{\vect{y} : \|\vect{y}_k\| = u_k, \forall k} \frac{1}{2}\sum_{k=1}^p\|\vect{y}_k - \vect{x}_k\|^2  \nonumber\\
&=& \min_{\vect{u}\in \set{R}^p_+}  \psi(u_1,\ldots,u_p) + \frac{1}{2}\sum_{k=1}^p \min_{\vect{y}_k : \|\vect{y}_k\| = u_k} \|\vect{y}_k - \vect{x}_k\|^2 \,\,\,\,(^*) \nonumber\\
&=& \min_{\vect{u}\in \set{R}^p_+}  \psi(u_1,\ldots,u_p) + \frac{1}{2}\sum_{k=1}^p\left\|\frac{u_k}{\|\vect{x}_k\|}\vect{x}_k - \vect{x}_k\right\|^2 \nonumber\\
&=& \min_{\vect{u}\in \set{R}^p_+}  \psi(u_1,\ldots,u_p) + \frac{1}{2}\sum_{k=1}^p(u_k - \|\vect{x}_k\|)^2 \nonumber\\
&=& M_{\psi}(\|\vect{x}_1\|,\ldots,\|\vect{x}_p\|),
\end{eqnarray}
where the solution of the innermost minimization problem in $(^*)$ is $\vect{y}_k = \frac{u_k}{\|\vect{x}_k\|}\vect{x}_k$, and therefore
$[\mathrm{prox}_{\varphi}(\vect{x}_1,\ldots,\vect{x}_p)]_k = [\mathrm{prox}_{\psi}(\|\vect{x}_1\|,\ldots,\|\vect{x}_p\|)]_k \frac{\vect{x}_k}{\|\vect{x}_k\|}$.

\section{Proof of Corollary~\ref{cor:bound2}}\label{sec:proof_cor_bound2}

We start by stating and proving the following lemma:
\begin{lemma}\label{lemma:bound1}
Let $\varphi:\set{R}^p \rightarrow \bar{\set{R}}$ be as in Prop.~\ref{prop:moreau},
and let $\bar{\vect{x}} \triangleq \mathrm{prox}_{\varphi}(\vect{x})$.
Then, any $\vect{y} \in \set{R}^p$ satisfies
\begin{eqnarray}\label{eq:bound1}
(\bar{\vect{x}} - \vect{y})^\top (\bar{\vect{x}} - \vect{x}) \le
\varphi(\vect{y}) - \varphi(\bar{\vect{x}})
\end{eqnarray}
\end{lemma}
\begin{proof}
From \eqref{eq:moreau}, we have that
\begin{eqnarray}
\frac{1}{2}\|\vect{x}\|^2
&=& \frac{1}{2}\|\bar{\vect{x}} - \vect{x}\|^2 + \varphi(\bar{\vect{x}}) + \frac{1}{2}\|\bar{\vect{x}}\|^2 + \varphi^*(\vect{x}-\bar{\vect{x}}) \nonumber\\
&=& \frac{1}{2}\|\bar{\vect{x}} - \vect{x}\|^2 + \varphi(\bar{\vect{x}}) + \frac{1}{2}\|\bar{\vect{x}}\|^2 + \sup_{\vect{u} \in \set{R}^p} \left( \vect{u}^\top(\vect{x}-\bar{\vect{x}}) - \varphi(\vect{u}) \right) \nonumber\\
&\ge& \frac{1}{2}\|\bar{\vect{x}} - \vect{x}\|^2 + \varphi(\bar{\vect{x}}) + \frac{1}{2}\|\bar{\vect{x}}\|^2 + \vect{y}^\top(\vect{x}-\bar{\vect{x}}) - \varphi(\vect{y}) \nonumber\\
&=& \frac{1}{2}\|\vect{x}\|^2 + {\bar{\vect{x}}}^\top(\bar{\vect{x}} - \vect{x}) + \vect{y}^\top(\vect{x}-\bar{\vect{x}}) - \varphi(\vect{y})+ \varphi(\bar{\vect{x}})  \nonumber\\
&=& \frac{1}{2}\|\vect{x}\|^2 + ({\bar{\vect{x}}} - \vect{y})^\top(\bar{\vect{x}} - \vect{x}) - \varphi(\vect{y})+ \varphi(\bar{\vect{x}}),\nonumber
\end{eqnarray}
from which \eqref{eq:bound1} follows.
\end{proof}

Now, take Lemma~\ref{lemma:bound1} and bound the left hand side as:
\begin{eqnarray}
(\bar{\vect{x}} - \vect{y})^\top (\bar{\vect{x}} - \vect{x}) &\ge& (\bar{\vect{x}} - \vect{y})^\top (\bar{\vect{x}} - \vect{x}) -
\frac{1}{2}\|\bar{\vect{x}} - \vect{x}\|^2 \nonumber\\
&=& (\bar{\vect{x}} - \vect{y})^\top (\bar{\vect{x}} - \vect{x}) - \frac{1}{2}\|\bar{\vect{x}}\|^2 - \frac{1}{2}\|\vect{x}\|^2 +
{\bar{\vect{x}}}^\top \vect{x} \nonumber\\
&=& \frac{1}{2}\|\bar{\vect{x}}\|^2 - \vect{y}^\top (\bar{\vect{x}} - \vect{x}) - \frac{1}{2}\|\vect{x}\|^2\nonumber\\
&=& \frac{1}{2}\|\vect{y} - \bar{\vect{x}}\|^2 - \frac{1}{2}\|\vect{y} - \vect{x}\|^2.\nonumber
\end{eqnarray}
This concludes the proof.

\section{Proof of Lemma~\ref{lemma:softprojgrad2}}\label{sec:proof_lemma_softprojgrad2}

Let $u(\bar{\vectsymb{\theta}},\vectsymb{\theta}) \triangleq \lambda R(\bar{\vectsymb{\theta}}) - \lambda R({\vectsymb{\theta}})$.
We have successively:
\begin{eqnarray}
\|\bar{\vectsymb{\theta}} - \vectsymb{\theta}_{t+1}\|^2
&\le^\text{(i)}& \|\bar{\vectsymb{\theta}} - \tilde{\vectsymb{\theta}}_{t+1}\|^2 \nonumber\\
&\le^\text{(ii)}& \|\bar{\vectsymb{\theta}} - \tilde{\vectsymb{\theta}}_t\|^2 +2\eta_t \lambda \sum_{j=1}^J (R_j(\bar{\vectsymb{\theta}}) - R_j({\tilde{\vectsymb{\theta}}_{t+j/J}})) \nonumber\\
&\le^\text{(iii)}& \|\bar{\vectsymb{\theta}} - \tilde{\vectsymb{\theta}}_t\|^2 +2\eta_t u(\bar{\vectsymb{\theta}},\tilde{\vectsymb{\theta}}_{t+1}) \nonumber\\
&\le^\text{(iv)}& \|\bar{\vectsymb{\theta}} - \tilde{\vectsymb{\theta}}_t\|^2 +2\eta_t u(\bar{\vectsymb{\theta}},\vectsymb{\theta}_{t+1}) \nonumber\\
&=& \|\bar{\vectsymb{\theta}} - \vectsymb{\theta}_t\|^2 + \|{\vectsymb{\theta}}_t - \tilde{\vectsymb{\theta}}_t\|^2 +
2(\bar{\vectsymb{\theta}} - \vectsymb{\theta}_t)^\top (\vectsymb{\theta}_t
- \tilde{\vectsymb{\theta}}_t) +2\eta_t u(\bar{\vectsymb{\theta}},\vectsymb{\theta}_{t+1}) \nonumber\\
&=& \|\bar{\vectsymb{\theta}} - \vectsymb{\theta}_t\|^2 + \eta_t^2 \|\vect{g}\|^2 + 2\eta_t(\bar{\vectsymb{\theta}} - \vectsymb{\theta}_t)^\top \vect{g}
+2\eta_t u(\bar{\vectsymb{\theta}},\vectsymb{\theta}_{t+1}) \nonumber\\
&\le^\text{(v)}& \|\bar{\vectsymb{\theta}} - \vectsymb{\theta}_t\|^2 + \eta_t^2 \|\vect{g}\|^2 +
2\eta_t (L(\bar{\vectsymb{\theta}}) - L(\vectsymb{\theta}_t)) +2\eta_t u(\bar{\vectsymb{\theta}},\vectsymb{\theta}_{t+1}) \nonumber\\
&\le& \|\bar{\vectsymb{\theta}} - \vectsymb{\theta}_t\|^2 + \eta_t^2 G^2 +
2\eta_t (L(\bar{\vectsymb{\theta}}) - L(\vectsymb{\theta}_t)) +2\eta_t u(\bar{\vectsymb{\theta}},\vectsymb{\theta}_{t+1}),
\end{eqnarray}
where the inequality (i) is due to the nonexpansiveness of the projection operator,
(ii) follows from applying Corollary~\ref{cor:bound2} $J$ times, (iii) follows from applying
the inequality $R_j(\tilde{\vectsymb{\theta}}_{t+l/J}) \ge R_j(\tilde{\vectsymb{\theta}}_{t+(l+1)/J})$
for $l = j,\ldots,J-1$,
(iv) results from the fact that $R(\tilde{\vectsymb{\theta}}_{t+1}) \ge R(\Pi_{\Theta}(\tilde{\vectsymb{\theta}}_{t+1}))$,
and (v) results from the subgradient inequality of convex functions, which has an extra term
$\frac{\sigma}{2}\|\bar{\vectsymb{\theta}} - \vectsymb{\theta}_t\|^2$ if $L$ is $\sigma$-strongly convex.

\section{Proof of Proposition~\ref{prop:softprojgrad2}}\label{sec:proof_prop_softprojgrad2}

Invoke Lemma~\ref{lemma:softprojgrad2} and sum for $t=1,\ldots,T$, which gives
\begin{eqnarray}
&&\sum_{t=1}^T \left( L(\vectsymb{\theta}_t;x_t,y_t) + \lambda R(\vectsymb{\theta}_t) \right)\nonumber\\
&=& \sum_{t=1}^T \left( L(\vectsymb{\theta}_t;x_t,y_t) + \lambda R(\vectsymb{\theta}_{t+1}) \right) -
\lambda (R(\vectsymb{\theta}_{m+1}) - R(\vectsymb{\theta}_1))\nonumber\\
&\le^\text{(i)}& \sum_{t=1}^T \left( L(\vectsymb{\theta}_t;x_t,y_t) + \lambda R(\vectsymb{\theta}_{t+1}) \right) \nonumber\\
&\le& \sum_{t=1}^T \left( L(\vectsymb{\theta}^*;x_t,y_t) + \lambda R(\vectsymb{\theta}^*) \right)  +
\frac{G^2 }{2} \sum_{t=1}^T \eta_t +
\sum_{t=1}^T\frac{\|\vectsymb{\theta}^* - \vectsymb{\theta}_t\|^2 - \|\vectsymb{\theta}^* - \vectsymb{\theta}_{t+1}\|^2}{2\eta_t}
\nonumber\\
&=& \sum_{t=1}^T \left( L(\vectsymb{\theta}^*;x_t,y_t) + \lambda R(\vectsymb{\theta}^*) \right)  +
\frac{G^2 }{2} \sum_{t=1}^T \eta_t +
\frac{1}{2}\sum_{t=2}^T\left( \frac{1}{\eta_{t}} - \frac{1}{\eta_{t-1}}\right) \cdot \|\vectsymb{\theta}^* - \vectsymb{\theta}_t\|^2
\nonumber\\
&& + \frac{1}{2\eta_{1}} \cdot \|\vectsymb{\theta}^* - \vectsymb{\theta}_1\|^2
- \frac{1}{2\eta_{T}} \cdot \|\vectsymb{\theta}^* - \vectsymb{\theta}_{T+1}\|^2
\end{eqnarray}
where the inequality (i) is due to the fact that $\vectsymb{\theta}_{1}=\vect{0}$.
Noting that the third term vanishes for a constant learning rate
and that the last term is non-positive suffices to prove the first part.
For the second part, we continue as:
\begin{eqnarray}
&&\sum_{t=1}^T \left( L(\vectsymb{\theta}_t;x_t,y_t) + \lambda R(\vectsymb{\theta}_t) \right)\nonumber\\
&\le & \sum_{t=1}^T \left( L(\vectsymb{\theta}^*;x_t,y_t) + \lambda R(\vectsymb{\theta}^*) \right)  +
\frac{G^2 }{2} \sum_{t=1}^T \eta_t +
\frac{F^2}{2}\sum_{t=2}^T\left( \frac{1}{\eta_{t}} - \frac{1}{\eta_{t-1}}\right)
+ \frac{F^2}{2\eta_{1}} \nonumber\\
&=& \sum_{t=1}^T \left( L(\vectsymb{\theta}^*;x_t,y_t) + \lambda R(\vectsymb{\theta}^*) \right)  +
\frac{G^2 }{2} \sum_{t=1}^T \eta_t
+ \frac{F^2}{2\eta_{T}} \nonumber\\
&\le^\text{(ii)}& \sum_{t=1}^T \left( L(\vectsymb{\theta}^*;x_t,y_t) + \lambda R(\vectsymb{\theta}^*) \right)  +
G^2 \eta_0 (\sqrt{T} - 1/2)
+ \frac{F^2 \sqrt{T}}{2\eta_0}\nonumber\\
&\le& \sum_{t=1}^T \left( L(\vectsymb{\theta}^*;x_t,y_t) + \lambda R(\vectsymb{\theta}^*) \right)  +
\left( G^2 \eta_0 
+ \frac{F^2 }{2\eta_0} \right) \sqrt{T},
\end{eqnarray}
where equality (ii) is due to the
fact that $\sum_{t=1}^T \frac{1}{\sqrt{t}} \le 2\sqrt{T} - 1$.
For the third part, continue after inequality (i) as:
\begin{eqnarray}
&&\sum_{t=1}^T \left( L(\vectsymb{\theta}_t;x_t,y_t) + \lambda R(\vectsymb{\theta}_t) \right)\nonumber\\
&\le& \sum_{t=1}^T \left( L(\vectsymb{\theta}^*;x_t,y_t) + \lambda R(\vectsymb{\theta}^*) \right)  +
\frac{G^2 }{2} \sum_{t=1}^T \eta_t +
\frac{1}{2}\sum_{t=2}^T\left( \frac{1}{\eta_{t}} - \frac{1}{\eta_{t-1}} - \sigma \right) \cdot \|\vectsymb{\theta}^* - \vectsymb{\theta}_t\|^2
\nonumber\\
&& + \frac{1}{2}\left(\frac{1}{\eta_{1}} - \sigma \right) \cdot \|\vectsymb{\theta}^* - \vectsymb{\theta}_1\|^2
- \frac{1}{2\eta_{T}} \cdot \|\vectsymb{\theta}^* - \vectsymb{\theta}_{T+1}\|^2
\nonumber\\
&=& \sum_{t=1}^T \left( L(\vectsymb{\theta}^*;x_t,y_t) + \lambda R(\vectsymb{\theta}^*) \right)  +
\frac{G^2 }{2 \sigma } \sum_{t=1}^T \frac{1}{t} - \frac{\sigma T}{2} \cdot \|\vectsymb{\theta}^* - \vectsymb{\theta}_{T+1}\|^2
\nonumber\\
&\le& \sum_{t=1}^T \left( L(\vectsymb{\theta}^*;x_t,y_t) + \lambda R(\vectsymb{\theta}^*) \right)  +
\frac{G^2 }{2 \sigma } \sum_{t=1}^T \frac{1}{t}
\nonumber\\
&\le^\text{(iii)}& \sum_{t=1}^T \left( L(\vectsymb{\theta}^*;x_t,y_t) + \lambda R(\vectsymb{\theta}^*) \right)  +
\frac{G^2 }{2 \sigma } (1 + \log T),
\end{eqnarray}
where the equality (iii) is due to the
fact that $\sum_{t=1}^T \frac{1}{t} \le 1 + \log T$.

\section{Lipschitz Constants of Some Loss Functions}\label{sec:lipschitz}

Let $\vectsymb{\theta}^*$ be a solution of
the problem \eqref{eq:learningproblem} with $\Theta = \set{R}^d$.
For certain loss functions, we may obtain bounds of the form $\|\vectsymb{\theta}^*\| \le \gamma$ for some $\gamma > 0$,
as the next proposition illustrates. Therefore, we may redefine
$\Theta = \{\vectsymb{\theta} \in \set{R}^d \,\,|\,\, \|\vectsymb{\theta}\| \le \gamma\}$ (a vacuous constraint)
without affecting the solution of \eqref{eq:learningproblem}.
\begin{proposition}
Let $R(\vectsymb{\theta}) = \frac{1}{2}(\sum_{k=1}^p\|\vectsymb{\theta}_k\|)^2$.
Let $L_{\mathrm{SVM}}$ and $L_{\mathrm{CRF}}$ be the structured hinge and logistic losses \eqref{eq:LSVM}.
Assume that the average cost function (in the SVM case) or the average entropy (in the CRF case) are bounded by some
$\Lambda \ge 0$, i.e.,%
\footnote{In sequence binary labeling, we have $\Lambda = \bar{N}$ for the CRF case and for the SVM case with a Hamming cost function,
where $\bar{N}$ is the average sequence length. Observe that the entropy of a distribution over labelings of a sequence of length $N$
is upper bounded by $\log 2^N = N$.} %
\begin{eqnarray}
\frac{1}{m}\sum_{i=1}^m \max_{y_i'\in  \in \sett{Y}(x_t)} \ell(y_i';y_i) \le \Lambda \quad \text{or} \quad
\frac{1}{m}\sum_{i=1}^m H(Y_i) \le \Lambda.
\end{eqnarray}
Then:
\begin{enumerate}
\item The solution of \eqref{eq:learningproblem} with $\Theta = \set{R}^d$
satisfies $\|\vectsymb{\theta}^*\| \le \sqrt{2\Lambda/\lambda}$.
\item $L$ is $G$-Lipschitz on $\set{R}^d$, with
$G = 2\max_{u\in\sett{U}} \|\vectsymb{\phi}(u)\|$.
\item Consider the following problem obtained from \eqref{eq:learningproblem} by adding a quadratic term:
\begin{equation}
\min_{\vectsymb{\theta}} \frac{\sigma}{2}\|\vectsymb{\theta}\|^2 + \lambda R(\vectsymb{\theta}) + \frac{1}{m} \sum_{i=1}^m L(\vectsymb{\theta};x_i,y_i).
\end{equation}
The solution of this problem satisfies $\|\vectsymb{\theta}^*\| \le
\sqrt{2\Lambda/(\lambda + \sigma)}$.
\item The modified loss $\tilde{L} = L + \frac{\sigma}{2}\|.\|^2$ is $\tilde{G}$-Lipschitz on
$\left\{\vectsymb{\theta} \,\, | \,\, \|\vectsymb{\theta}\| \le \sqrt{2\Lambda/(\lambda + \sigma)}\right\}$,
where $\tilde{G} = G + \sqrt{2\sigma^2 \Lambda / (\lambda + \sigma)}$.
\end{enumerate}
\end{proposition}
\begin{proof}
Let $F_{\mathrm{SVM}}(\vectsymb{\theta})$ and $F_{\mathrm{CRF}}(\vectsymb{\theta})$ be the objectives of \eqref{eq:learningproblem} for
the SVM and CRF cases. We have
\begin{eqnarray}
F_{\mathrm{SVM}}(\vect{0}) &=&
\lambda R(\vect{0}) + \frac{1}{m}\sum_{i=1}^m L_{\mathrm{SVM}}(\vect{0}; x_i, y_i) = \frac{1}{m}\sum_{i=1}^m \max_{y_i' \in \sett{Y}(x_i)} \ell(y_i';y_i)
\le \Lambda_{\mathrm{SVM}}\\
F_{\mathrm{CRF}}(\vect{0}) &=& \lambda R(\vect{0}) + \frac{1}{m}\sum_{i=1}^m L_{\mathrm{CRF}}(\vect{0}; x_i, y_i) = \frac{1}{m}\sum_{i=1}^m \log |\sett{Y}(x_i)|
\le \Lambda_{\mathrm{CRF}}
\end{eqnarray}
Using the facts that $F(\vectsymb{\theta}^*) \le F(\vect{0})$, that the losses are non-negative,
and that $(\sum_i |x_i|)^2 \ge \sum_i x_i^2$, we obtain
$\frac{\lambda}{2}\|\vectsymb{\theta}^*\|^2 \le \lambda R(\vectsymb{\theta}^*) \le F(\vectsymb{\theta}^*) \le F(\vect{0})$,
which proves the first statement.

To prove the second statement for the SVM case, note that 
a subgradient of $L_{\mathrm{SVM}}$ at $\vectsymb{\theta}$ is 
$\vect{g}_{\mathrm{SVM}} = \vectsymb{\phi}(x,\hat{y}) - \vectsymb{\phi}(x,y)$,
where $\hat{y} = \arg\max_{y' \in \sett{Y}(x)} \vectsymb{\theta}^\top (\vectsymb{\phi}(x,y') - \vectsymb{\phi}(x,y)) + \ell(y';y)$; 
and that the gradient of $L_{\mathrm{CRF}}$ at $\vectsymb{\theta}$ is 
$\vect{g}_{\mathrm{CRF}} = \mathbb{E}_{\vectsymb{\theta}} \vectsymb{\phi}(x,Y) - \vectsymb{\phi}(x,y)$.
Applying Jensen's inequality, we have that
$\|\vect{g}_{\mathrm{CRF}}\| \le \mathbb{E}_{\vectsymb{\theta}} \|\vectsymb{\phi}(x,Y) - \vectsymb{\phi}(x,y)\|$.
Therefore, both $\|\vect{g}_{\mathrm{SVM}}\|$ and $\|\vect{g}_{\mathrm{CRF}}\|$ are
upper bounded by $\max_{x\in\sett{X},y,y'\in\sett{Y}(x)} \|\vectsymb{\phi}(x,y') - \vectsymb{\phi}(x,y)\| \le 2\max_{u\in\sett{U}} \|\vectsymb{\phi}(u)\|$.

The same rationale can be used to prove the third and fourth statements.
\end{proof}

\section{Computing the proximity operator of the (non-separable) squared $\ell_1$}\label{sec:proximoper_sql1}

We present an algorithm (Alg.~\ref{alg:softproj2}) that computes the Moreau projection of the \emph{squared}, \emph{weighted} $\ell_1$-norm.
Denote by $\odot$ the Hadamard product, $[\vect{a} \odot \vect{b}]_k = a_k b_k$. 
Letting $\lambda,\vect{d} \ge 0$, and $\phi_{\vect{d}}(\vect{x}) \triangleq \frac{1}{2} \|\vect{d} \odot \vect{x}\|_1^2$,
the underlying optimization problem is:
\begin{equation}\label{eq:proximsquaredl1}
M_{\lambda \phi_{\vect{d}}}(\vect{x}_0) \triangleq \min_{\vect{x} \in \set{R}^p} \frac{1}{2}\|\vect{x} - \vect{x}_0\|^2 + \frac{\lambda}{2} \left( \sum_{i=1}^p d_i |x_i| \right)^2.
\end{equation}
This includes the squared $\ell_1$-norm as a particular case, when $\vect{d} = \vect{1}$ (the case addressed in Alg.~\ref{alg:softproj}).
The proof is somewhat technical and follows the same procedure employed by \citet{Duchi2008} to derive an algorithm for
projecting onto the $\ell_1$-ball. The runtime is $O(p\log p)$ (the amount of time
that is necessary to sort the vector), but a similar trick as the one described in \citep{Duchi2008}
can be employed to yield $O(p)$ runtime.

\begin{algorithm}[t]
   \caption{Moreau projection for the squared weighted $\ell_1$-norm}
\begin{algorithmic}\label{alg:softproj2}
   \STATE {\bfseries Input:} A vector $\vect{x}_0 \in \set{R}^p$, a weight vector $\vect{d} \ge 0$, and a parameter $\lambda > 0$
   \STATE Set $u_{0r} = |x_{0r}|/d_r$ and $a_r = d_r^2$ for each $r=1,\ldots,p$
   \STATE Sort $\vect{u}_0$: $u_{0(1)} \ge \ldots \ge u_{0(p)}$
   \STATE Find $\rho = \max \left\{j \in \{1,\ldots,p\} \,\,|\,\, u_{0(j)} - \frac{\lambda}{1 + \lambda \sum_{r=1}^j a_{(r)}} \sum_{r=1}^j a_{(r)} u_{0(r)} > 0\right\}$
   \STATE Compute $\vect{u} = \mathrm{soft}(\vect{u}_0, \tau)$, where $\tau = \frac{\lambda}{1 + \lambda \sum_{r=1}^\rho a_{(r)}} \sum_{r=1}^{\rho} a_{(r)} u_{0(r)}$
   \STATE {\bfseries Output:} $\vect{x}$ s.t. $x_r = \mathrm{sign}(x_{0r}) d_r u_r$.
\end{algorithmic}
\end{algorithm}

\begin{lemma}\label{lemma:longproof1}
Let $\vect{x}^* = \mathrm{prox}_{\lambda \phi_{\vect{d}}}(\vect{x}_0)$ be
the solution of \eqref{eq:proximsquaredl1}. Then:
\begin{enumerate}
\item $\vect{x}^*$ agrees in sign with $\vect{x}_0$, \emph{i.e.},
each component satisfies $x_{0i}\cdot x_i^* \ge 0$.
\item Let $\vectsymb{\sigma} \in \{-1,1\}^p$. Then
$\mathrm{prox}_{\lambda \phi_{\vect{d}}}(\vectsymb{\sigma} \odot \vect{x}_0) = \vectsymb{\sigma} \odot \mathrm{prox}_{\lambda \phi_{\vect{d}}}(\vect{x}_0)$,
\emph{i.e.}, flipping a sign in $\vect{x}_0$ produces a $\vect{x}^*$ with the same sign flipped.
\end{enumerate}
\end{lemma}
\begin{proof}
Suppose that $x_{0i}\cdot x_i^* < 0$ for some $i$.
Then, $\vect{x}$ defined by $x_j = x_{j}^*$ for $j \ne i$ and $x_i = -x_{i}^*$
achieves a lower objective value than $\vect{x}^*$,
since $\phi_{\vect{d}}(\vect{x}) = \phi_{\vect{d}}(\vect{x}^*)$
and $(x_i - x_{0i})^2 < (x_i^* - x_{0i})^2$; this contradicts the optimality of $\vect{x}^*$.
The second statement is a simple consequence of the first one and that
$\phi_{\vect{d},\lambda}(\vectsymb{\sigma} \odot  \vect{x}) = \phi_{\vect{d},\lambda}(\vectsymb{\sigma} \odot  \vect{x}^*)$.
\end{proof}

\smallskip

Lemma~\ref{lemma:longproof1} enables reducing the problem to the non-negative orthant,
by writing $\vect{x}_0 = \vectsymb{\sigma} \cdot \tilde{\vect{x}}_0$, with $\tilde{\vect{x}}_0 \ge \vect{0}$,
obtaining a solution $\tilde{\vect{x}}^*$ and then recovering the true solution as $\vect{x}^* = \vectsymb{\sigma} \cdot \tilde{\vect{x}}^*$.
It therefore suffices to solve \eqref{eq:proximsquaredl1} with the constraint $\vect{x} \ge \vect{0}$,
which in turn can be transformed into:
\begin{eqnarray}\label{eq:proximsquaredl1posorth}
\min_{\vect{u} \ge \vect{0}} F(\vect{u}) \triangleq \frac{1}{2}\sum_{r=1}^p a_r (u_r - u_{0r})^2 + \frac{\lambda}{2} \left( \sum_{r=1}^p a_r u_r \right)^2,
\end{eqnarray}
where we made the change of variables $a_i \triangleq d_i^2$, $u_{0i} \triangleq x_{0i}/d_i$ and $u_i \triangleq x_i/d_i$.

The Lagrangian of \eqref{eq:proximsquaredl1posorth} is
$\mathcal{L}(\vect{u},\vectsymb{\xi}) = \frac{1}{2}\sum_{r=1}^p a_r (u_r - u_{0r})^2 + \frac{\lambda}{2} \left( \sum_{r=1}^p a_r u_r \right)^2 - \vectsymb{\xi}^{\top} \vect{u}$,
where $\vectsymb{\xi} \ge \vect{0}$ are Lagrange multipliers. Equating the gradient (w.r.t. $\vect{u}$) to zero gives
\begin{equation}\label{eq:softproj_complslack}
\vect{a} \odot (\vect{u} - \vect{u}_0) + \lambda \sum_{r=1}^p a_r u_r \vect{a} - \vectsymb{\xi} = \vect{0}.
\end{equation}
From the complementary slackness condition, $u_j > 0$ implies $\xi_j = 0$, which in turn implies
\begin{equation}\label{eq:softproj_complslack2}
a_j (u_j - u_{0j}) + \lambda a_j \sum_{r=1}^p a_r u_r  = 0. 
\end{equation}
Thus, if $u_j > 0$, the solution is of the form $u_j = u_{0j} - \tau$, with $\tau = \lambda \sum_{r=1}^p a_r u_r$.
The next lemma shows the existence of a split point below which some coordinates vanish.

\begin{lemma}
Let $\vect{u}^*$ be the solution of \eqref{eq:proximsquaredl1posorth}.
If $u_k^*=0$ and $u_{0j} < u_{0k}$, then we must have $u_j^*=0$.
\end{lemma}
\begin{proof}
Suppose that
$u_j^*=\epsilon>0$. We will construct a $\tilde{\vect{u}}$ whose objective value is lower than $F(\vect{u}^*)$,
which contradicts the optimality of $\vect{u}^*$: set $\tilde{u}_l = u_l^*$ for $l \notin \{j,k\}$,
$\tilde{u}_k = \epsilon c$, and $\tilde{u}_j = \epsilon \left(1 - c a_k/a_j\right)$, where $c = \min\{a_j/a_k, 1\}$.
We have $\sum_{r=1}^p a_r u_r^* = \sum_{r=1}^p a_r \tilde{u}_r$, and therefore
\begin{eqnarray}\label{eq:lemmasquaredl1_1}
2(F(\tilde{\vect{u}}) - F(\vect{u}^*)) &=& \sum_{r=1}^p a_r (\tilde{u}_r - u_{0r})^2 - \sum_{r=1}^p a_r (u_r^* - u_{0r})^2\nonumber\\
&=& a_j (\tilde{u}_j - u_{0j})^2 - a_j (u_j^* - u_{0j})^2
+ a_k (\tilde{u}_k - u_{0k})^2 - a_k (u_k^* - u_{0k})^2.
\end{eqnarray}
Consider the following two cases: (i) if $a_j \le a_k$, then
$\tilde{u}_k = \epsilon a_j/a_k$ and $\tilde{u}_j = 0$. Substituting in \eqref{eq:lemmasquaredl1_1},
we obtain
$2(F(\tilde{\vect{u}}) - F(\vect{u}^*)) = \epsilon^2 \left( a_j^2/a_k - a_j\right) \le 0$, which leads
to the contradiction $F(\tilde{\vect{u}}) \le F(\vect{u}^*)$. If (ii) $a_j > a_k$, then
$\tilde{u}_k = \epsilon$ and $\tilde{u}_j = \epsilon \left(1 - a_k/a_j\right)$.
Substituting in \eqref{eq:lemmasquaredl1_1},
we obtain
$2(F(\tilde{\vect{u}}) - F(\vect{u}^*)) = a_j \epsilon^2 \left( 1 - a_k/a_j\right)^2
+ 2a_k \epsilon u_{0j} - 2a_k \epsilon u_{0k} + a_k \epsilon^2 - a_j \epsilon^2
< a_k^2/a_j \epsilon^2 - 2a_k \epsilon^2 + a_k \epsilon^2 =
\epsilon^2 \left( a_k^2/a_j - a_k\right) < 0$, which also leads to a contradiction.
\end{proof}

\bigskip

Let $u_{0(1)}\ge \ldots \ge u_{0(p)}$ be the entries of $\vect{u}_0$ sorted in decreasing order,
and let $u^*_{(1)},\ldots,u^*_{(p)}$ be the entries
of $\vect{u}^*$ under the same permutation. Let $\rho$ be the number of nonzero entries in $\vect{u}^*$ , i.e.,
$u^*_{(\rho)} > 0$, and, if $\rho < p$, $u^*_{(\rho+1)} = 0$.
Summing \eqref{eq:softproj_complslack2} for $(j) = 1,\ldots, \rho$, we get
\begin{equation}\label{eq:softproj_complslack3}
\sum_{r=1}^{\rho} a_{(r)} u^*_{(r)} - \sum_{r=1}^{\rho} a_{(r)} u_{0(r)} +
\left(\sum_{r=1}^{\rho} a_{(r)} \right)\lambda \sum_{r=1}^{\rho} a_{(r)} u^*_{(r)} = 0,
\end{equation}
which implies
\begin{equation}\label{eq:softproj_complslack4}
\sum_{r=1}^{p} u^*_{r} = \sum_{r=1}^{\rho} u^*_{(r)} = \frac{1}{1+\lambda \sum_{r=1}^{\rho} a_{(r)}}\sum_{r=1}^{\rho} a_{(r)} u_{0(r)},
\end{equation}
and therefore $\tau = \frac{\lambda}{1+\lambda \sum_{r=1}^{\rho} a_{(r)}}\sum_{r=1}^{\rho} a_{(r)} u_{0(r)}$.
The complementary slackness conditions for $r = \rho$ and $r = \rho+1$ imply
\begin{eqnarray}
u^*_{(\rho)} - u_{0(\rho)} + \lambda \sum_{r=1}^{\rho} a_{(r)} u^*_{(r)} = 0 \quad \text{and} \quad
- u^*_{0(\rho+1)} + \lambda \sum_{r=1}^{\rho} a_{(r)} u^*_{(r)} = \xi_{(\rho+1)}\ge 0;
\end{eqnarray}
therefore $u_{0(\rho)} > u_{0(\rho)} - u^*_{(\rho)} = \tau \ge u_{0(\rho+1)}$.
This implies that $\rho$ is such that
\begin{equation}
u_{0(\rho)} > \frac{\lambda}{1 + \lambda \sum_{r=1}^{\rho} a_{(r)}} \sum_{r=1}^\rho a_{(r)} u_{0(r)} \ge u_{0(\rho+1)}.
\end{equation}
The next proposition goes farther by exactly determining $\rho$.

\begin{proposition}
The quantity $\rho$ can be determined via:
\begin{equation}
\rho = \max \left\{j \in [p] \,\,\Big|\,\, u_{0(j)} - \frac{\lambda}{1 + \lambda \sum_{r=1}^{j} a_{(r)}}  \sum_{r=1}^j a_{(r)} u_{0(r)} > 0\right\}.
\end{equation}
\end{proposition}

\begin{proof}
Let $\rho^* = \max\{j | u^*_{(j)} > 0\}$. We have that
$u^*_{(r)} = u_{0(r)} - \tau^*$ for $r \le \rho^*$, where $\tau^* = \frac{\lambda}{1+\lambda \sum_{r=1}^{\rho^*} a_{(r)}}\sum_{r=1}^{\rho^*} a_{(r)} u_{0(r)}$,
and therefore $\rho \ge \rho^*$. We need to prove that $\rho \le \rho^*$, which we will do by contradiction.
Assume that $\rho > \rho^*$. Let $\vect{u}$ be the vector induced by the choice of $\rho$, i.e.,
$u_{(r)} = 0$ for $r > \rho$ and $u_{(r)} = u_{0(r)} - \tau$ for $r \le \rho$,
where $\tau = \frac{\lambda}{1+\lambda \sum_{r=1}^{\rho} a_{(r)}}\sum_{r=1}^{\rho} a_{(r)} u_{0(r)}$.
From the definition of $\rho$, we have $u_{(\rho)} = u_{0(\rho)} - \tau > 0$, which implies $u_{(r)} = u_{0(r)} - \tau > 0$
for each $r \le \rho$.
In addition,
\begin{eqnarray}
\sum_{r=1}^p a_r u_r &=& \sum_{r=1}^{\rho} a_{(r)} u_{0(r)} - \sum_{r=1}^{\rho} a_{(r)} \tau
= \left( 1 - \frac{\lambda \sum_{r=1}^{\rho} a_{(r)}}{1 + \lambda \sum_{r=1}^{\rho} a_{(r)}}\right) \sum_{r=1}^{\rho} a_{(r)} u_{0(r)}\nonumber\\
&=& \frac{1}{1 + \lambda \sum_{r=1}^{\rho} a_{(r)}} \sum_{r=1}^{\rho} a_{(r)} u_{0(r)}
= \frac{\tau}{\lambda},\label{eq:lemma_final_squaredl1_1} \\
\sum_{r=1}^p a_r (u_r - u_{0r})^2 &=& \sum_{r=1}^{\rho^*} a_{(r)} \tau^2 + \sum_{r=\rho^*+1}^{\rho} a_{(r)} \tau^2
+ \sum_{r=\rho+1}^{p} a_{(r)} u_{0(r)}^2 \nonumber\\
&<& \sum_{r=1}^{\rho^*} a_{(r)} \tau^2 + \sum_{r=\rho^*+1}^{p} a_{(r)} u_{0(r)}^2.\label{eq:lemma_final_squaredl1_2}
\end{eqnarray}
We next consider two cases:

\framebox{$\tau^* \ge \tau$.}
From \eqref{eq:lemma_final_squaredl1_2}, we have that
$\sum_{r=1}^p a_r (u_r - u_{0r})^2 < \sum_{r=1}^{\rho^*} a_{(r)} \tau^2 + \sum_{r=\rho^*+1}^{p} a_{(r)} u_{0(r)}^2
\le \sum_{r=1}^{\rho^*} a_{(r)} (\tau^*)^2 + \sum_{r=\rho^*+1}^{p} a_{(r)} u_{0(r)}^2
= \sum_{r=1}^p a_r (u_r^* - u_{0r})^2$.
From \eqref{eq:lemma_final_squaredl1_1}, we have that
$\left( \sum_{r=1}^p a_r u_r \right)^2 = \tau^2/\lambda^2 \le (\tau^*)^2/\lambda^2$.
Summing the two inequalities, we get $F(\vect{u}) < F(\vect{u}^*)$, which leads to a contradiction.

\framebox{$\tau^* < \tau$.}
We will construct a vector $\tilde{\vect{u}}$ from $\vect{u}^*$ and show that
$F(\tilde{\vect{u}}) < F(\vect{u}^*)$. Define
\begin{equation}
\tilde{u}_{(r)} = \left\{
\begin{array}{ll}
u^*_{(\rho^*)} - \frac{2 a_{(\rho^* + 1)}}{a_{(\rho^*)} + a_{(\rho^* + 1)}}\epsilon, & \text{if $r = \rho^*$}\\
\frac{2 a_{(\rho^*)}}{a_{(\rho^*)} + a_{(\rho^* + 1)}}\epsilon, & \text{if $r = \rho^* + 1$}\\
u^*_{(r)} & \text{otherwise,}
\end{array}
\right.
\end{equation}
where $\epsilon = (u_{0(\rho^* + 1)} - \tau^*)/2$.
Note that $\sum_{r=1}^p a_r \tilde{u}_r = \sum_{r=1}^p a_r u^*_r$.
From the assumptions that $\tau^* < \tau$ and $\rho^* < \rho$, we have that
$u^*_{(\rho^* + 1)} = u_{0(\rho^* + 1)} - \tau > 0$, which implies that
$\tilde{u}_{(\rho^* + 1)} = \frac{a_{(\rho^*)} (u_{0(\rho^* + 1)} - \tau^*)}{a_{(\rho^*)} + a_{(\rho^* + 1)}} >
\frac{a_{(\rho^*)} (u_{0(\rho^* + 1)} - \tau)}{a_{(\rho^*)} + a_{(\rho^* + 1)}}
= \frac{a_{(\rho^*)} u^*_{(\rho^* + 1)}}{a_{(\rho^*)} + a_{(\rho^* + 1)}} > 0$, and
that $u^*_{(\rho^*)} = u_{0(\rho^*)} - \tau^* - \frac{a_{(\rho^* + 1)} (u_{0(\rho^* + 1)} - \tau^*)}{a_{(\rho^*)} + a_{(\rho^* + 1)}}
= u_{0(\rho^*)} - \frac{a_{(\rho^* + 1)} u_{0(\rho^* + 1)}}{a_{(\rho^*)} + a_{(\rho^* + 1)}}
- \left( 1 - \frac{a_{(\rho^* + 1)}}{a_{(\rho^*)} + a_{(\rho^* + 1)}}\right)\tau^*
>^{\text{(i)}} \left( 1 - \frac{a_{(\rho^* + 1)}}{a_{(\rho^*)} + a_{(\rho^* + 1)}}\right)( u_{0(\rho^* + 1)} - \tau)
= \left( 1 - \frac{a_{(\rho^* + 1)}}{a_{(\rho^*)} + a_{(\rho^* + 1)}}\right)( u^*_{(\rho^* + 1)}) > 0$,
where inequality (i) is justified by the facts that $u_{0(\rho^*)} \ge u_{0(\rho^* + 1)}$ and $\tau > \tau^*$.
This ensures that $\tilde{\vect{u}}$ is well defined.
We have:
\begin{eqnarray}\label{eq:lemma_final_squaredl1_3}
2(F(\vect{u}^*) - F(\tilde{\vect{u}})) &=&
\sum_{r=1}^p a_r (u^*_r - u_{0r})^2 - \sum_{r=1}^p a_r (\tilde{u}_r - u_{0r})^2 \nonumber\\
&=&
a_{(\rho^*)} (\tau^*)^2 + a_{(\rho^*+1)} u_{0(\rho^* + 1)}^2
- a_{(\rho^*)} \left( \tau^* + \frac{2 a_{(\rho^*+1)} \epsilon}{a_{(\rho^*)} + a_{(\rho^*+1)}}\right)^2 \nonumber\\ &&
- a_{(\rho^*+1)} \left( u_{0(\rho^* + 1)} - \frac{2 a_{(\rho^*)} \epsilon}{a_{(\rho^*)} + a_{(\rho^*+1)}}\right)^2 \nonumber\\
&=&
-\frac{4 a_{(\rho^*)} a_{(\rho^*+1)} \epsilon}{a_{(\rho^*)} + a_{(\rho^*+1)}} \underbrace{(\tau^* - u_{0(\rho^* + 1)})}_{-2\epsilon}
-\frac{4 a_{(\rho^*)} a_{(\rho^*+1)}^2 \epsilon^2}{\left(a_{(\rho^*)} + a_{(\rho^*+1)}\right)^2}
-\frac{4 a_{(\rho^*)}^2 a_{(\rho^*+1)} \epsilon^2}{\left(a_{(\rho^*)} + a_{(\rho^*+1)}\right)^2} \nonumber\\
&=&
\frac{4 a_{(\rho^*)} a_{(\rho^*+1)} \epsilon^2}{a_{(\rho^*)} + a_{(\rho^*+1)}} \ge 0,
\end{eqnarray}
which leads to a contradiction and completes the proof.
\end{proof}

\vskip 0.2in
\bibliographystyle{apalike}
\bibliography{OnlineMKLSP2010}

\end{document}